\documentclass{article}

\usepackage[T1]{fontenc}
\usepackage[utf8]{inputenc}
\usepackage[american]{babel}
\usepackage[backend = bibtex, style = alphabetic, sorting = ynt, citestyle = authoryear]{biblatex} 
\DeclareCiteCommand{\cite}
  {\usebibmacro{prenote}}
  {\usebibmacro{citeindex}%
   \printtext[bibhyperref]{\usebibmacro{cite}}}
  {\multicitedelim}
  {\usebibmacro{postnote}}

\DeclareCiteCommand*{\cite}
  {\usebibmacro{prenote}}
  {\usebibmacro{citeindex}%
   \printtext[bibhyperref]{\usebibmacro{citeyear}}}
  {\multicitedelim}
  {\usebibmacro{postnote}}

\DeclareCiteCommand{\parencite}[\mkbibparens]
  {\usebibmacro{prenote}}
  {\usebibmacro{citeindex}%
    \printtext[bibhyperref]{\usebibmacro{cite}}}
  {\multicitedelim}
  {\usebibmacro{postnote}}

\DeclareCiteCommand*{\parencite}[\mkbibparens]
  {\usebibmacro{prenote}}
  {\usebibmacro{citeindex}%
    \printtext[bibhyperref]{\usebibmacro{citeyear}}}
  {\multicitedelim}
  {\usebibmacro{postnote}}

\DeclareCiteCommand{\footcite}[\mkbibfootnote]
  {\usebibmacro{prenote}}
  {\usebibmacro{citeindex}%
  \printtext[bibhyperref]{ \usebibmacro{cite}}}
  {\multicitedelim}
  {\usebibmacro{postnote}}

\DeclareCiteCommand{\footcitetext}[\mkbibfootnotetext]
  {\usebibmacro{prenote}}
  {\usebibmacro{citeindex}%
   \printtext[bibhyperref]{\usebibmacro{cite}}}
  {\multicitedelim}
  {\usebibmacro{postnote}}

\DeclareCiteCommand{\textcite}
  {\boolfalse{cbx:parens}}
  {\usebibmacro{citeindex}%
   \printtext[bibhyperref]{\usebibmacro{textcite}}}
  {\ifbool{cbx:parens}
     {\bibcloseparen\global\boolfalse{cbx:parens}}
     {}%
   \multicitedelim}
  {\usebibmacro{textcite:postnote}}

\usepackage[ruled]{algorithm2e}

\usepackage{amsmath}
\usepackage{latexsym}
\usepackage{amsthm}
\usepackage{amsfonts}
\usepackage{amssymb}
\usepackage{mathrsfs}
\usepackage{dsfont}
\usepackage{bbm}
\usepackage{graphicx}
\usepackage{longtable}
\usepackage{parskip}
\usepackage{enumerate}
\usepackage{undertilde}
\usepackage{hyperref}
\usepackage{wrapfig}
\usepackage{multirow}
\usepackage{hhline}
\usepackage{tabularx}
\usepackage{adjustbox}
\usepackage{pdflscape}
\usepackage{rotating}
\usepackage{array}
\usepackage{tikz}
\usepackage{float}
\usepackage{fancyvrb}
\usepackage{csquotes}
\usepackage{verbatim}
\usepackage{color}
\usepackage{soul}
\usepackage[hmargin=1in,vmargin=1in]{geometry}
\definecolor{blue}{HTML}{0000EE}
\definecolor{black}{HTML}{000000}
\hypersetup{
  linkcolor  = blue,
  citecolor  = blue,
  urlcolor   = black,
  colorlinks = true,
}
\usepackage[
  open,
  openlevel=2,
  atend,
  numbered
]{bookmark}

\usepackage{cleveref}

\crefformat{section}{\S#2#1#3} 
\crefformat{subsection}{\S#2#1#3}
\crefformat{subsubsection}{\S#2#1#3}
\usepackage{afterpage}

\def\RR{\mathds{R}}

\def\QQ{\mathds{Q}}
\def\EE{\mathds{E}}

\def\PP{\mathds{P}}

\def\BBB{\mathscr{B}}

\def\DDD{\mathscr{D}}

\def\LLL{\mathscr{L}}

\def\NNN{\mathscr{N}}

\def\ind{\mathds{1}}

\def\wrt{with respect to }

\newcommand{\cmmnt}[1]{\ignorespaces}

\numberwithin{equation}{section}
\newtheorem{lem}{Lemma}
\newtheorem{thm}{Theorem}
\newtheorem{cor}{Corollary}
\newtheorem{cond}{Condition}

\begin{document}

\title{Boosting Random Forests to Reduce Bias;\\ One-Step Boosted Forest and its Variance Estimate}
\author{\textbf{Indrayudh Ghosal} \\ {\small \href{mailto:IG248@CORNELL.EDU}{\nolinkurl{IG248@CORNELL.EDU}}} \and \textbf{Giles Hooker} \\ {\small \href{mailto:GJH27@CORNELL.EDU}{\nolinkurl{GJH27@CORNELL.EDU}}} \and Department of Statistical Science, Cornell University \\ Ithaca, NY 14850, USA}
\maketitle

\begin{abstract}
In this paper we propose using the principle of boosting to reduce the bias of a random forest prediction in the regression setting. From the original random forest fit we extract the residuals and then fit another random forest to these residuals. We call the sum of these two random forests a \textit{one-step boosted forest}. We show with simulated and real data that the one-step boosted forest has a reduced bias compared to the original random forest. The paper also provides a variance estimate of the one-step boosted forest by an extension of the infinitesimal Jackknife estimator. Using this variance estimate we can construct prediction intervals for the boosted forest and we show that they have good coverage probabilities. Combining the bias reduction and the variance estimate we show that the one-step boosted forest has a significant reduction in predictive mean squared error and thus an improvement in predictive performance. When applied on datasets from the UCI database, one-step boosted forest performs better than  random forest and gradient boosting machine algorithms. Theoretically we can also extend such a boosting process to more than one step and the same principles outlined in this paper can be used to find variance estimates for such predictors. Such boosting will reduce bias even further but it risks over-fitting and also increases the computational burden.
\end{abstract}

\section{Introduction}

Ensemble methods have become one of the most successful and widely-used methods in machine learning. Ensembles of trees, in particular, have the advantage of being computationally fast and of having few tuning parameters and requiring minimal human intervention [\cite{breiman2001random}, \cite{friedman2001elements}]. These methods can be classified into two categories: ``bagging-type'' methods which reduce variance by combining trees that are obtained using identical randomised processes, and ``boosting-type'' methods which grow trees sequentially, one tree depending on the output of the previous. Recent work in \cite{mentch2016quantifying} and \cite{wager2017estimation} has demonstrated a central limit theorem for random forests -- a bagging-type method -- allowing for uncertainty quantification about its predictions.  In this paper, we leverage this to take a step towards boosting methods. We revisit a bias correction method for regression originally proposed in \cite{breiman2001random}  and further studied in \cite{zhang2012bias} and \cite{Xu2013}: we build two random forests, the second obtained from the residuals of the first, and then add them together. This represents one step of gradient boosting, as examined in \cite{friedman2001greedy} for squared-error regression and we name the resulting algorithm {\em One-step Boosted Forests}. While the method is not novel, it has not been widely recognised within statistical learning despite near universal improvement in test set accuracy. In particular, it does better than either random forests or gradient boosting in experiments on data from the UCI repository [\cite{lichman2013uci}]. In this paper, we build on recent work in \cite{mentch2016quantifying} and \cite{wager2017estimation} to develop variance estimates, show asymptotic normality and hence confidence intervals for the resulting predictions when ensemble methods are built using subsamples of the data.

Random forests [\cite{breiman2001random}] and other ensemble methods have proven to be one of the most popular machine learning methods [\cite{fernandez2014we}]. They proceed by generating trees from bootstrap or subsamples of the available data, potentially incorporating additional randomisation within the tree building process. By averaging many trees built in this fashion, random forests achieve a considerable reduction in variance relative to any individual tree. More recently, this structure has been re-interpreted in the framework of U-statistics [\cite{van2000asymptotic}] allowing \cite{mentch2016quantifying} and \cite{wager2017estimation} to develop central limit theorems for the resulting predictions. Crucially, the variance of these predictions, and hence confidence intervals for them, can be calculated at no additional computational effort.

Despite their successes, random forests can suffer from bias when targeting complex signals. Since each tree is generated using an identical random process, they cannot be used to compensate each other. In particular, in a complex signal, each tree will target the same part of the signal, potentially leaving a bias {\em that could have been effectively modelled by a random forest, if it were the only target of estimation}. This is a result of the nonlinear use of data in the tree-building algorithm, as well as a tree partitioning scheme which quickly reduces the amount of data available to model local features. As an alternative, boosting methods build trees sequentially, allowing the current tree to correct for the biases of those that were constructed before it. Boosting was originally developed for classification in \cite{freund1995desicion} and \cite{schapire2012boosting}. In the regression setting \cite{friedman2001greedy} developed gradient boosting; in the context of least-squares regression, each tree is built to predict the residuals from the current model. In order to reduce over-fitting, gradient boosting introduces a shrinkage parameter and sets the number of trees (boosting steps) as a tuning parameter that must be chosen by a validation set.

One-step boosted forests combine these approaches. By fitting the residuals of a random forest, we are able to reduce the bias of the random forest procedure. By using the already-stable forest procedure we remove the need for shrinkage. The procedure can be iterated, but while our experiments suggest that while the first round of boosting makes a large improvement in predictive performance, subsequent boosting iterations might provide marginal benefit. The details of this conclusion depend on the amount of noise in the data. In all our experiments, One-step boosted forests outperformed both random Forests and gradient boosting in test-set performance. We explain this in terms of their ability to improve the bias of random forests while providing greater stability than gradient boosting. They also outperformed the bias-corrected random forests described in [\cite{hooker2015bootstrap}], which were designed with similar principles in mind.

An important aspect of this paper is the extension of the variance calculations in \cite{mentch2016quantifying} and \cite{wager2017estimation} to one-step boosted forests. We show that the infinitesimal Jackknife estimator in \cite{wager2017estimation} can be extended to the joint covariance of both random forests in our one-step boosting procedure, and therefore for their sum. Under a couple of regularity assumptions, the two forests are also jointly normal, justifying the development of confidence intervals. There are two potential variants of these forests depending on whether the subsamples used for both forests (the base step and the boosting step) are the same or not. Our empirical results suggest that taking different (independent) subsamples reduces the over-all variance and hence improves prediction more than repeating the same subsamples. This theoretical framework is also extensible to multiple rounds of boosting, thereby also providing an alternative means of assessing when to stop.  While not examined here, some of the bias correction variations described in \cite{zhang2012bias} could also be analysed using similar techniques. 

The layout of this paper is as follows: in \cref{sec:basicalgo} we give define our estimate, and in \cref{sec:theory} we will showcase some properties of the estimate, such as its theoretical variance as well as demonstrating asymptotic normality in \cref{sec:asympnormal} followed with an estimate for the theoretical variance (and its consistency)  in \cref{sec:varxest}. \cref{sec:discussions} continues the discussion with comparisons to other works in the literature and possible extensions. In \cref{sec:results} we empirically demonstrate the utility of our algorithm - \cref{sec:simresults} focuses on results from a simulated dataset and \cref{sec:realresults} demonstrates the performance of our algorithm on some datasets from the UCI database (\cite{lichman2013uci}). In the latter the boosted forest is compared against the basic random forest algorithm and also against other standard algorithms such as gradient boosting machine [\cite{friedman2001greedy}] and bootstrap bias-correction [\cite{hooker2015bootstrap}].

\section{Defining One-Step Boosted Forests} \label{sec:basicalgo}

We first set some notation used throughout this paper. Let $Z_{[n]}^{(0)} = (Z_1^{(0)}, Z_2^{(0)}, \dots, Z_n^{(0)})$ denote the dataset, where $Z_i^{(0)} = (Y_i, X_i)$ and $[n]$ is the set $\{1,\dots,n\}$. Further note that random forests were originally described with bootstrap samples, but implementations also allow subsampling and using subsamples of size $k < n$ has been important in making theoretical progress. In this paper we will refer to these types of random forests as usual ones, unless specified otherwise. 

We now formally define the method of creating the One-Step Boosted Forest. We build two forests, the first being the usual random forest - if there are $B$ trees in the forest then for each of them we select $k (<n)$ datapoints at random without replacement and denote by $T(x; Z^{(0)}_I)$ the estimate from that tree for any test point $x$, where $I$ is the set of $k$ indices selected (i.e., $|I| = k$). Let $I_1^{(0)}, \dots, I_B^{(0)}$ be the indices selected for each of the trees in the forest, with each of the $I_b^{(0)}$ having $k$ elements. Then the estimate after the first stage will be
\begin{equation}
\hat{F}^{(0)}(x) = \frac1B \sum_{b=1}^B T(x; Z_{I_b^{(0)}}^{(0)}) \label{eqn:firstrf}
\end{equation}
Since selection of the subsets $I$ are random, we can assign random weights $w_I^{(0)}$ to each of the ${n \choose k}$ possible subsets. Each $w_I^{(0)}$ will be a binary random variable taking the value ${n \choose k}/B$ with probability $B/{n \choose k}$ and the value 0 with probability $1 - B/{n \choose k}$. The weights $w_I^{(0)}$ are then i.i.d random variables with $\EE(w_I^{(0)}) = 1$ and $c := Var(w_I^{(0)}) = {n \choose k}/B - 1$. Thus the formula for the random forest in \eqref{eqn:firstrf} can be rewritten as
\begin{equation}
\hat{F}^{(0)}(x) = \frac1{{n \choose k}} \sum_{I \subseteq [n]: |I| = k} w_I^{(0)} T(x; Z_{I}^{(0)}) \label{eqn:firststage}
\end{equation}
Note that this approach isn't exactly the same as taking $B$ subsets at random. Since $w_I^{(0)}$ are i.i.d the total number of trees has an \textbf{expected value} of $B$ and isn't always exactly equal to $B$. The difference between these two selection procedures has been shown to be ignorable in \cref{sec:moreresults}.

Once we have obtained $\hat{F}^{(0)}$ we can derive the residuals $e_i = Y_i - \hat{F}^{(0)}(X_i)$ and construct a new dataset $Z_{[n]}^{(1)} = (Z_1^{(1)}, \dots, Z_n^{(1)})$, where $Z_i^{(1)} = (e_i, X_i)$. We can repeat the same method as above on this dataset with weights $w_I^{(1)}$ and get a new estimate for the second stage
\begin{equation}
\hat{F}^{(1)}(x) = \frac1{{n \choose k}} \sum_{I \subseteq [n]: |I| = k} w_I^{(1)} T(x; Z_{I}^{(1)}) \label{eqn:secondstage}
\end{equation}
This is the second forest and the first boosted step in our algorithm. Our final estimate is the sum of the first and second stage estimates, \eqref{eqn:firststage} and \eqref{eqn:secondstage}, given by
\begin{equation}
\hat{F}(x) = \hat{F}^{(0)}(x) + \hat{F}^{(1)}(x) \label{eqn:bfest}
\end{equation}
For this construction we can take $w_I^{(1)}$ to be the same as or independent from $w_I^{(0)}$, i.e., choosing the same subsets in the second stage as the first stage or independent ones. Based on that choice our estimate will also change. We thus have to consider two variants in this paper:
\begin{itemize}
\item If $w_I^{(1)} = w_I^{(0)}$ then our estimate \eqref{eqn:bfest} is the \textit{one-step boosted forest with same subsamples}.
\item If $w_I^{(1)} \perp w_I^{(0)}$ then our estimate \eqref{eqn:bfest} is the \textit{one-step boosted forest with independent subsamples}.
\end{itemize}

Algorithms \ref{algo:boostforestsame} and \ref{algo:boostforestind} gives details of these two methods along with their variance estimates discussed in \cref{sec:varxest}. We will compare the performance of these variants in \cref{sec:simresults} \& \cref{sec:realresults}. In the next sections we shall try to quantify the variability and provide theoretical guarantees for our estimate $\hat{F}$ in \eqref{eqn:bfest}.

\section{Analysing One-Step Boosted Forests} \label{sec:theory}

In this section we will show that the one-step boosted forest can be expressed as a weighted U-statistic (with random weights). Using that fact we calculate the theoretical variance and also provide a central limit theorem for the one-step boosted forest. Then we also provide a variance estimate for the one-step boosted forest and we finally present the formal algorithm combining our findings. Before that we need a crucial assumption.

\paragraph{U-statistics} We first familiarise ourselves with U-statistics. If $h(z_1, \dots, z_k)$ is a symmetric function then the U-statistic $U$ with kernel $h$ is defined by
$$
U(Z_1, \dots, Z_n) = \frac1{{n \choose k}} \sum_{I \subseteq [n]: |I| = k} h(Z_I) = \frac1{{n \choose k}} \sum_{I \subseteq [n]: |I| = k} h(Z_{I_1}, \dots, Z_{I_k})
$$
\cite{hoeffding1948class} showed further that the U-statistic is asymptotically normal with variance $\frac{k^2}{n} \zeta_{1,k}$, where
\begin{align*}
\zeta_{c,k} &= cov(h(Z_1,\dots,Z_c,Z_{c+1},\dots,Z_k), h(Z_1,\dots,Z_c,Z'_{c+1},\dots,Z'_k)) \\
&= var(\EE(h(Z_1,\dots,Z_k) | Z_1 = z_1, \dots, Z_c = z_c))
\end{align*}

This result is stated in more detail in \cref{sec:proofmain}. 

\subsection{A Pivotal Assumption}\label{sec:assumption}

It is easily seen that the first stage forest in $\hat{F}^{(0)}(x)$ \eqref{eqn:firststage} can be thought of as a weighted complete U-statistic. But then note that $Z_i^{(1)}$ defined in \cref{sec:basicalgo} actually depends on the whole of the previous dataset $Z_{[n]}^{(0)}$ and so does $T(x; Z_I^{(1)})$ regardless of the subset $I$. Hence $T(x; Z_I^{(1)})$ is not a valid kernel for a U-statistic which makes $\hat{F}^{(1)}(x)$ not a valid U-statistic. However, if we replace $\hat{F}^{(1)}(x)$ in \eqref{eqn:secondstage} with $\check{F}^{(1)}(x)$ trained in the same manner but based on data with ``noise-free'' residuals:
$$
\check{Z}^{(1)}_i = \left(Y_i - \EE\left[\hat{F}^{(0)}(X_i)\right], X_i \right)
$$
then it is easily seen that $\check{F}^{(1)}(x)$ is a U-statistic. This will help in further analysis of boosted forest defined in \eqref{eqn:bfest} since it can be expressed as the sum of two (weighted) U-statistics. Note that $\check{F}^{(1)}(x)$ does not inherit variability from $\hat{F}^{(0)}(x)$, although the two will still be correlated. This approximation leads to significantly simplified analysis. Throughout this section and the corresponding proofs in the appendix we will assume that the following condition holds true.

\begin{cond} \label{cond:regularity}
Let
$$
\check{F}^{(1)}(x) = \frac1{{n \choose k}} \sum_{I \subseteq [n]: |I| = k} w_I^{(1)} T(x; \check{Z}_{I}^{(1)})
$$
then
$$
\frac{\hat{F}^{(1)}(x) - \check{F}^{(1)}(x)}{\sqrt{var(\hat{F}^{(1)}(x))}}  \xrightarrow{p} 0.
$$
\end{cond}
That is, the effect of the variance in $\hat{F}^{(0)}(x)$ on $e_i$ is ignorable in $\hat{F}^{(1)}(x)$. Throughout, our theoretical analysis will apply to $\check{F}(x) = \hat{F}^{(0)}(x) + \check{F}^{(1)}(x)$. From Condition \ref{cond:regularity} it is seen that asymptotic variance for $\hat{F}$ and $\check{F}$ will be the same, as will their asymptotic distributions.

This condition is crucial since it allows us to apply the theory of U-statistics (especially asymptotic normality) to $\check{F}(x)$ and be sure that it also works for $\hat{F}(x)$. Whether this condition applies in practice depends on the details of the tree building procedure. The true response function $F^*(x)$ will influence the tree structure in the first stage which in turn influences the residuals and then the tree structure of the second stage. All of these influences are difficult to quantify and we do not attempt a full analysis here. In practice condition \ref{cond:regularity} may not hold for all possible tree/forest building procedures but in \cref{sec:kernel} we consider an analogy with kernel methods as explored in \cite{scornet2016random}. There we show that this property holds for Nadaraya-Watson estimators if the bandwidth in the second stage is smaller than in the first: approximately corresponding to using deeper trees with smaller leaf sizes for $\hat{F}^{(1)}(x)$ compared to $\hat{F}^{(0)}(x)$. The property also holds if the two stages use different sets of covariates without any restriction on bandwidth relationships. As a heuristic our condition should hold when the trees that comprise $\hat{F}^{(1)}(x)$ (estimating the bias) tend to have a different set of splits than those in $\hat{F}^{(0)}(x)$ (the original estimate of the signal). This will, of course, depend on the specific algorithm employed to create the trees, as well as the properties of the underlying response function; and an analysis of such specific cases is beyond the scope of this paper. An empirical evaluation of the regularity condition is looked into briefly in \cref{sec:regularplot}.

As a further check on the validity of this assumption, we provide a detailed examination of the sample distribution of the predictions of the procedure in \cref{sec:simresults} where we see empirical confirmation of our results and good coverage of prediction intervals.

For the rest of this paper we shall try to restrict usage of the check ( $\check{ }$ ) accent to reduce notational complexity. For theoretical calculations in \cref{sec:asympnormal} we shall use the notations from \cref{sec:basicalgo}, eg., $\hat{F}$ to denote $\check{F}$, $Z$ to denote $\check{Z}$, etc. For empirical procedures described in \cref{sec:varxest} and \cref{sec:results} onwards we don't have access to $\check{Z}$, etc so $\hat{F}$ will denote the usual boosted forest defined in \cref{sec:basicalgo}.


\subsection{Asymptotic Normality}\label{sec:asympnormal}

In this section we will prove that the One-Step Boosted Forest predictions are asymptotically normal. Later in \cref{sec:varxest} we show that the theoretical asymptotic variance is consistently estimated by \eqref{eqn:bfsamevarest} and \eqref{eqn:bfindvarest} for the two variants respectively.

To calculate the theoretical asymptotic variance of the boosted forest \eqref{eqn:bfest} we will first condition over the weights to get a complete U-statistic. Then we follow the result about complete U-statistics at the beginning of \cref{sec:theory}. Also note that $\zeta_{k,k} = var(h(Z_1,\dots,Z_k))$ follows from the above definition of $\zeta_{c,k}$. To make calculations simpler, we shall assume, without loss of generality, that the individual trees and thus the random forest on both stages have zero mean, i.e., $\EE [T(x; Z_I)] = 0$ for all $I \subseteq [n]: |I| = k$. As discussed near the end of \cref{sec:basicalgo} there are two variants whose asymptotic variances can be calculated to be
\begin{align}
V_{same}(x) &:= \frac{k^2}{n}\zeta_{1,k} + \frac{c}{{n \choose k}} \zeta_{k,k} \label{eqn:bfsamevar} \\
V_{ind}(x) &:= \frac{k^2}{n}(\zeta^{(0)}_{1,k} + \zeta^{(1)}_{1,k} + 2\zeta^{(0,1)}_{1,k}) + \frac{c}{{n \choose k}} (\zeta^{(0)}_{k,k} + \zeta^{(1)}_{k,k}) \label{eqn:bfindvar}
\end{align}
Here the $\zeta$ values for $V_{same}$ are based on the kernel being $T(x; Z_{I}^{(0)}) + T(x; Z_{I}^{(1)})$, i.e., sum of the trees in the two stages rather than the individual trees for separate stages. The $\zeta^{(0)}$ and $\zeta^{(1)}$ values for $V_{ind}$ are based on the kernels being $T(x; Z_{I}^{(0)})$ and $T(x; Z_{I}^{(1)})$ respectively. Also here $\zeta^{(0,1)}_{1,k}$ is the (scaled) covariance term between the trees in the two stages, $T(x; Z_{I}^{(0)})$ and $T(x; Z_{I}^{(1)})$. The details of this calculation can be found in \cref{sec:Uvarest}.


Now it is shown in \cite{van2000asymptotic} that U-statistics are asymptotically normal if $k$, in this case the subsample size for random forests, is constant as $n \to \infty$. If we assume that here then asymptotic normality of the boosted forest follows. But in practice as $n$ increases we want $k$ to increase as well. To allow for this we need to assume the following Lindeberg-Feller type condition (initially presented as Condition 1 in \cite{mentch2016quantifying})

\begin{cond}\label{cond:mainLF}
Assume that the dataset $Z_1,Z_2,\dots \overset{iid}{\sim} D_Z$ and let $T(Z_1,\dots,Z_{k_n})$ be the tree kernel based on a subsample of size $k_n$. Define $T_{1,k_n}(z) = \EE T(z,Z_2,\dots,Z_{k_n})$. Then we assume that for all $\delta > 0$
$$
\lim_{n \to \infty} \frac1{\zeta_{1,k_n}} \EE\left[ T^2_{1,k_n}(Z_1) \ind\{|T_{1,k_n}(Z_1)| > \delta\zeta_{1,k_n}\} \right] = 0
$$ 
\end{cond}
A discussion of the situation in which this condition holds can be found in \cite{mentch2016quantifying}, in particular they argue that a sub-Guassian response, combined with a limit to the influence individual data points on a the values of a tree are sufficient for this to hold. Using the regularity conditions \ref{cond:regularity} and \ref{cond:mainLF} we can prove the following result

\begin{thm}\label{thm:normal}
Assume that the dataset $Z_1^{(0)}, Z_2^{(0)}, \dots \overset{iid}{\sim} D_{Z^{(0)}}$ and that $\EE T^4(x; Z_1,\dots,Z_{k_n}) \leq C < \infty$ for all $x, n$ and some constant $C$. Let $B_n$ be the number of trees in each step of the One-Step Boosted Forest. Then
as long as $k_n, B_n \to \infty$ such that $\frac{k_n}{n} \to 0$ and $\frac{n}{B_n} \to 0$ as $n \to \infty$ as well as $\lim\limits_{n \to \infty} \frac{k_n \zeta_{1,k_n}}{\zeta_{k_n,k_n}} \neq 0$ we have
$$
\frac{\hat{F}(x)}{\sigma_n(x)} \overset{\DDD}{\to} \NNN(0,1)
$$
for some sequence $\sigma_n(x)$ given by $\sigma_n^2(x) = V_{same}(x)$ from \eqref{eqn:bfsamevar} for Variant I of the One-Step Boosted Forest or $\sigma_n^2(x) = V_{ind}(x)$ from \eqref{eqn:bfindvar} for Variant II of the One-Step Boosted Forest.
\end{thm}

\subsection{Variance Estimation of the One-Step Boosted Forest}\label{sec:varxest}

Now that we have the formulae for the theoretical variance of both variants we can go about finding estimates for them. We will find estimates for each term in \eqref{eqn:bfsamevar} and \eqref{eqn:bfindvar} separately. In this section for simplicity we define $T^{(j)}_b(x) = T(x; Z_{I_b^{(j)}}^{(j)})$ for $j = 0,1$.

Note that the $\zeta_{k,k}$ values can be estimated by just the variability of the individual trees in the forests, by adding them up for \textbf{Variant I} and separately for \textbf{Variant II}. As an example
\begin{equation}
\widehat{\zeta^{(0)}_{k,k}} = var_*[T^{(0)}_b(x)], b = 1,\dots,B \label{eqn:zetakest}
\end{equation}
Here $var_*$ is used to denote empirical variance by varying $b = 1,\dots,B$. We shall use the same notation (the subscript $()_*$) for the rest of this paper.

Now note that $\frac{k^2}{n}\zeta_{1,k}$ is the variance for a random forest when we consider \textbf{all} possible subsets of the dataset of size $k$, i.e., a complete U-statistic. As an example if we define
\begin{equation}
\tilde{F}^{(0)}(x) = \frac1{{n \choose k}} \sum_{I \subseteq [n]: |I| = k} T(x; Z_{I}^{(0)}) \label{eqn:Ftildedef}
\end{equation}
then $\frac{k^2}{n}\zeta^{(0)}_{1,k} = var(\tilde{F}^{(0)}(x))$. From Theorem 1 [or Theorem 9] of \cite{wager2017estimation} we know that an asymptotically consistent estimate is given by the infinitesimal Jackknife estimator, the formula for which is
\begin{equation}
\widehat{var}_{IJ}(\tilde{F}^{(0)}(x)) = \sum_{i=1}^n cov_*\left[N_{i,b}^{(0)}, \,T^{(0)}_b(x)\right]^2, \;\; b = 1,\dots,B \label{eqn:zeta1est}
\end{equation}
as given by Theorem 1 of \cite{efron2014estimation} where $cov_*$ indicates the empirical covariance over $b$. Here $N_{i,b}^{(0)} = \ind\{i \in I^{(0)}_b\}$ is the indicator of whether the $i$\textsuperscript{th} datapoint is included in the calculations for the $b$\textsuperscript{th} tree. So we can estimate the variance for \textbf{Variant I} in \eqref{eqn:bfsamevar} by using equivalent expressions to \eqref{eqn:zeta1est} for the first term and \eqref{eqn:zetakest} for the second term:
\begin{align}
\hat{V}_{same}(x) &= \hat{V}_{IJ} + (1/B)*\hat{\zeta}_{k,k} \nonumber\\
&= \sum_{i=1}^n cov_*\left[ N_{i,b}^{(0)}, \,T^{(0)}_b(x)+T^{(1)}_b(x) \right]^2 + \frac1B \cdot var_*\left[T^{(0)}_b(x) + T^{(1)}_b(x)\right] \label{eqn:bfsamevarest}
\end{align}
In this formula we used the approximation $\frac{c}{{n \choose k}} = \frac1{{n \choose k}} \cdot \left(\frac{{n \choose k}}{B} - 1\right) = \frac1B - \frac1{{n \choose k}} \approx \frac1B$.

For the variance estimate of Variant II of the One-Step Boosted Forest we first need to find an estimate for $\frac{k^2}{n} \zeta^{(0,1)}_{1,k}$, the covariance between the first and second stages of our estimate in case of Variant II. It is reasonable to expect that we can have a two-sample analog of \eqref{eqn:zeta1est}, i.e., an infinitesimal Jackknife estimate for the covariance given by
\begin{equation}
\widehat{cov}_{IJ}(\tilde{F}^{(0)}(x),\tilde{F}^{(1)}(x)) = \sum_{i=1}^n cov_*\left[N_{i,b}^{(0)}, \,T^{(0)}_b(x)\right] \cdot cov_*\left[N_{i,b}^{(1)}, \,T^{(1)}_b(x)\right], \;\; b = 1,\dots,B \label{eqn:zeta1covest}
\end{equation}
The consistency of this estimate is proved in Lemma \ref{lem:IJest} in \cref{sec:proofIJ}. We can now estimate the variance for \textbf{Variant II} in \eqref{eqn:bfindvar} by using equivalent expressions to \eqref{eqn:zeta1est} and \eqref{eqn:zeta1covest} for the first term and \eqref{eqn:zetakest} for the second term. We get
\begin{align}
&\hat{V}_{ind}(x) = \hat{V}_{IJ} + (1/B)*\hat{\zeta}_{k,k} \nonumber \\
=& \sum_{i=1}^n \left( cov_*\left[ N_{i,b}^{(0)}, \,T^{(0)}_b(x) \right] + cov_*\left[ N_{i,b}^{(1)}, \,T^{(1)}_b(x) \right] \right)^2 + \frac1B \left( var_*\left[T^{(0)}_b(x)\right] + var_*\left[T^{(1)}_b(x)\right] \right) \nonumber \\
=& \sum_{i=1}^n \left( \sum_{j=0}^1  cov_*\left[ N_{i,b}^{(j)}, \, T^{(j)}_b(x) \right] \right)^2 + \frac1B\sum_{j=0}^1 \left( var_*\left[T^{(j)}_b(x)\right] \right) \label{eqn:bfindvarest}
\end{align}

Note that we still use $\frac{c}{{n \choose k}} \approx \frac1B$. Thus we have found variance estimates for the one-step boosted forest \eqref{eqn:bfest} formalised in \cref{sec:basicalgo}. We have the following result regarding these estimates.


\begin{thm}\label{thm:consistency}
The variance estimates discussed above are consistent:
\begin{itemize}
\item $V_{same}(x)$ in \eqref{eqn:bfsamevar} is consistently estimated by $\hat{V}_{same}(x)$ in \eqref{eqn:bfsamevarest}.
\item $V_{ind}(x)$ in \eqref{eqn:bfindvar} is consistently estimated by $\hat{V}_{ind}(x)$ in \eqref{eqn:bfindvarest}.
\end{itemize}
\end{thm}


The proof of follows directly from Lemma \ref{lem:IJest} in the appendix and the fact that the sample variance is a consistent estimator of the population variance.


\subsection{The Complete One-Step Boosted Forest Algorithm}

We present below complete algorithms for construction and variance estimation for Variant I [Algorithm \ref{algo:boostforestsame}] and Variant II [Algorithm \ref{algo:boostforestind}] of the One-Step Boosted Forest. The performance of both variants are compared among themselves and to other standard algorithms in \cref{sec:simresults} and \cref{sec:realresults}.

\begin{algorithm}[ht] \label{algo:boostforestsame}
\caption{One-Step Boosted Forest \textbf{Variant I} (\textit{Same subsets in both stages})}
\SetKwInOut{Input}{Input}
\SetKwInOut{Output}{Output}
\SetKwInOut{Obtain}{Obtain}
\SetKwInOut{Calculate}{Calculate}

\Input{The data $\big(Z_i^{(0)} = (Y_i, X_i)\big)_{i=1}^n$, the tree function $T$, the number of trees in the forest $B$, the subsample size for each tree $k$, and the test point $x$.}

\For{$b = 1$ \KwTo $B$}{
Choose $I^{(0)}_b \subseteq [n]$ randomly such that $|I^{(0)}_b| = k$.

\Calculate{$T^{(0)}_b(x) = T\left(x; Z_{I^{(0)}_b}^{(0)}\right)$ and $N_{i,b}^{(0)} = \ind\{i \in I^{(0)}_b\}$.}
}
\Obtain{The first stage estimate $\hat{F}^{(0)}(x) = \frac1B \sum_{b=1}^B T^{(0)}_b(x)$.}
Calculate residuals $e_i = Y_i - \hat{F}^{(0)}(X_i)$ and new dataset $\big(Z_i^{(1)} = (e_i, X_i)\big)_{i=1}^n$.

\For{$b = 1$ \KwTo $B$}{
Choose $I^{(1)}_b = I^{(0)}_b$.

\Calculate{$T^{(1)}_b(x) = T\left(x; Z_{I^{(1)}_b}^{(1)}\right)$.}
}
\Obtain{The second stage estimate $\hat{F}^{(1)}(x) = \frac1B \sum_{b=1}^B T^{(1)}_b(x)$.}
\Calculate{The first term of the variance estimate $\hat{V}_{IJ} = \sum_{i=1}^n cov_*\left[ N_{i,b}^{(0)}, \,T^{(0)}_b(x)+T^{(1)}_b(x) \right]^2$}
\Calculate{The (unscaled) second term of the variance estimate $\hat{\zeta}_{k,k} = var_*\left[T^{(0)}_b(x) + T^{(1)}_b(x)\right]$}
\Output{The \textit{one-step boosted forest} estimate at the test point $x$ given by $\hat{F}(x) = \hat{F}^{(0)}(x) + \hat{F}^{(1)}(x)$ and the variance estimate given by $\hat{V}_{same}(x) = \hat{V}_{IJ} + (1/B)*\hat{\zeta}_{k,k}$.}
\end{algorithm}

\begin{algorithm}[ht]
\caption{One-Step Boosted Forest \textbf{Variant II} (\textit{Independent subsets in 2 stages})} \label{algo:boostforestind}
\SetKwInOut{Input}{Input}
\SetKwInOut{Output}{Output}
\SetKwInOut{Obtain}{Obtain}
\SetKwInOut{Calculate}{Calculate}

\Input{The data $\big(Z_i^{(0)} = (Y_i, X_i)\big)_{i=1}^n$, the tree function $T$, the number of trees in the forest $B$, the subsample size for each tree $k$, and the test point $x$.}

\For{$b = 1$ \KwTo $B$}{
Choose $I^{(0)}_b \subseteq [n]$ randomly such that $|I^{(0)}_b| = k$.

\Calculate{$T^{(0)}_b(x) = T\left(x; Z_{I^{(0)}_b}^{(0)}\right)$ and $N_{i,b}^{(0)} = \ind\{i \in I^{(0)}_b\}$.}
}
\Obtain{The first stage estimate $\hat{F}^{(0)}(x) = \frac1B \sum_{b=1}^B T^{(0)}_b(x)$.}
Calculate residuals $e_i = Y_i - \hat{F}^{(0)}(X_i)$ and new dataset $\big(Z_i^{(1)} = (e_i, X_i)\big)_{i=1}^n$.

\For{$b = 1$ \KwTo $B$}{
Choose $I^{(1)}_b \subseteq [n]$ randomly such that $|I^{(1)}_b| = k$, i.e., an independent copy of the first stage subset.

\Calculate{$T^{(1)}_b(x) = T\left(x; Z_{I^{(1)}_b}^{(1)}\right)$ and $N_{i,b}^{(1)} = \ind\{i \in I^{(1)}_b\}$.}
}
\Obtain{The second stage estimate $\hat{F}^{(1)}(x) = \frac1B \sum_{b=1}^B T^{(1)}_b(x)$.}
\Calculate{The first term of the variance estimate $\hat{V}_{IJ} = \sum_{i=1}^n \left( cov_*\left[ N_{i,b}^{(0)}, \,T^{(0)}_b(x) \right] + cov_*\left[ N_{i,b}^{(1)}, \,T^{(1)}_b(x) \right] \right)^2$}
\Calculate{The (unscaled) second term of the variance estimate $\hat{\zeta}_{k,k} = var_*\left[T^{(0)}_b(x)\right] + var_*\left[T^{(1)}_b(x)\right]$}
\Output{The \textit{one-step boosted forest} estimate at the test point $x$ given by $\hat{F}(x) = \hat{F}^{(0)}(x) + \hat{F}^{(1)}(x)$ and the variance estimate given by $\hat{V}_{ind}(x) = \hat{V}_{IJ} + (1/B)*\hat{\zeta}_{k,k}$.}
\end{algorithm}

\section{Further discussions}\label{sec:discussions}

\subsection{Comparison with Prior Results in the literature}


Theorem \ref{thm:normal} is an extension and combination of previous work. Lemma 2 [or Theorem 1(i)] of \cite{mentch2016quantifying} and Theorem 1 of \cite{wager2017estimation} shows that $\hat{F}^{(0)}(x)$ has an asymptotically (possibly biased in the former paper) normal distribution; the former listing a variance of $\frac{k^2}{n}\zeta^{(0)}_{1,k}$ while the latter employs the infinitesimal Jackknife estimator as a consistent variance estimate. However, the two papers use different assumptions to demonstrate normality. We have used the conditions for the former result, but note that inspection of their proof of Theorem 1 (see page 29 of \cite{mentch2016quantifying}) allows a replacement of their conditions -- $k_n = o(\sqrt{n})$ and $\lim\limits_{n \to \infty} \zeta_{1,k_n} \neq 0$ -- with those we give above; see \cite{peng2019asymptotic}. \cite{wager2017estimation} requires $\frac{k_n (\log n)^d}{n} \to 0$ along with some conditions on the tree building process, but demonstrates that the bias in the resulting estimators is asymptotically ignorable. Either set of conditions could be employed within our result. 



Since we had assumed that the tree function $T$ has zero mean, our central limit theorem is actually centered on $\EE[\hat{F}(x)]$, but we could add the honesty assumption from \cite{wager2017estimation} (detailed in Lemma 2 and Theorem 3 of that paper) to change the center to be the target function $F(x)$. Note in that case the second boosting stage $\hat{F}^{(1)}(x)$ is asymptotically estimating 0. Now boosting is supposed to reduce the bias $\EE[\hat{F}(x)] - F(x)$, and the high empirical values of performance improvement (due to low values of MSE) in \cref{sec:simresults} suggests that in this case the honesty assumption might not be necessary in practice.

We can also get a similar result about the joint distribution of each stage of Variant II of the boosted forest, under the extra condition than $\lim\limits_{n \to \infty} (\zeta^{(1)}_{1,k_n}/\zeta^{(0)}_{1,k_n}) \notin \{0, \infty\}$. This will be a more general result compared to the above main theorem, and we can use any linear combination of the boosting steps to arrive at the final estimate rather than just adding them. This result (Theorem \ref{thm:main}) and its proof is in \cref{sec:proofnormal}.

Our variance estimates discussed in \cref{sec:varxest} borrow from the infinitesimal Jackknife estimate used in \cite{wager2017estimation} where there is an assumption that the number of trees, (i.e., the number of times we subsample) $B$ be so large as to negate Monte Carlo effects, i.e, large $B$ leads to $\widehat{var}^B_{IJ}$ being close to $\widehat{var}_{IJ} = \widehat{var}^\infty_{IJ}$. However, our theoretical variance formulae in \eqref{eqn:bfsamevar} and \eqref{eqn:bfindvar} accounts for this with an additional term. We thus use the infinitesimal Jackknife approach to only estimate $\frac{k^2}{n}\zeta_{1,k}$ (the first term in our formulae for the variance) and add an estimate for the second term. We also remove the finite sample correction factor $\frac{n(n-1)}{(n-k)^2}$ discussed in \cite{wager2017estimation} and the additive correction term in \cite{wager2014confidence}. Our simulation results below demonstrate an upward bias of the infinitesimal jackknife estimator, particularly for small $B$; we have found that standard correction terms in \cite{wager2014confidence} often result in negative variance estimates, see \cite{zhou2019asymptotic} for a discussion. 

Our boosting method corresponds to the method BC3 in \cite{zhang2012bias}; other bias correction methods in that paper also incorporate the response within a correction term. When the correction is given by a random forest (BC1 and BC3 in \cite{zhang2012bias}) our central limit theorem continues to hold. When correcting for response bias via smoothing splines (method BC2), the same conditions would require an analysis of the variance due to both random forests and splines. 

The boosted forest algorithm is unlike the bootstrap bias correction method in \cite{hooker2015bootstrap}, where bias was directly estimated via the bootstrap, but which did not include a variance estimate for the bias corrected random forest. The algorithm in \cite{hooker2015bootstrap} is akin to a two-sample U-statistic but the dependency within the data and the residuals (on which the bootstrap is done) makes it harder to obtain a variance estimate via the infinitesimal Jackknife. However we speculate that the algorithms in \cite{mentch2016quantifying} can be used to find an estimate of the variance of the bias correction algorithm.

\subsection{Extensions: More than One Boosting Step}\label{sec:extensions}

We could continue with the boosting process and reduce the bias even further. For example if we boosted once more we would define $Z^{(2)}_i = (Y_i - \hat{F}^{(0)}(X_i)-\hat{F}^{(1)}(X_i), X_i)$ to be the dataset  for the third stage output $\hat{F}^{(2)}(x)$. Our final output would be the \textit{2-step boosted forest} given by
$$
\hat{F}(x) = \hat{F}^{(0)}(x) + \hat{F}^{(1)}(x) + \hat{F}^{(2)}(x)
$$
Its variance would depend of which variant of the original algorithm we use. If we used the same subsets to generate all three random forests then the variance would be consistently estimated by
$$
\hat{V}_{same}(x) = \sum_{i=1}^n cov_*\left[ N_{i,b}^{(0)}, \,\sum_{j=0}^2 T^{(j)}_b(x) \right]^2 + \frac1B \cdot var_*\left[\sum_{j=0}^2 T^{(j)}_b(x)\right]
$$
We could also use subsets independently generated for all three stages and then the variance estimate would be given by
$$
\hat{V}_{ind}(x) = \sum_{i=1}^n \left( \sum_{j=0}^2  cov_*\left[ N_{i,b}^{(j)}, \, T^{(j)}_b(x) \right] \right)^2 + \frac1B\sum_{j=0}^2 \left( var_*\left[T^{(j)}_b(x)\right] \right)
$$
We could also tweak the process and take independent subsets in the first two stages and then the same in the last stages, i.e., in terms of notation in \cref{sec:basicalgo} the weights could be $w^{(0)}_I, w^{(1)}_I, w^{(1)}_I$ respectively for the 3 stages. We could actually have two more combinations, namely $w^{(0)}_I, w^{(0)}_I, w^{(2)}_I$ and $w^{(0)}_I, w^{(1)}_I, w^{(0)}_I$. Thus there are 5 variants of the 2-step boosted forest based on these combinations and for each combination we can easily find out the variance estimates using the principles outlined in \cref{sec:varxest}.

For an $M$-step boosted forest we can easily see that the number of variants is given by $a_{M+1}$, where
$$
a_n = \sum_{k=1}^n a_{n,k} \;\;\;\text{with}\; a_{n,k} = ka_{n-1,k} + a_{n-1,k-1} \,\forall\, n>k \;\;\;\text{and}\; a_{k,k} = a_{n,1} = 1 \,\forall\, n,k
$$
For each of these variants the final estimate will simply be the sum of all the boosting steps and the variance can be found by following similar steps as outlined in in \cref{sec:varxest}.

\section{Empirical Studies for One-Step Boosted Forest} \label{sec:results}


We shall focus on performances of our algorithm in this section. Our implementation differs slightly from the theory above in the following ways.
\begin{itemize}
\item In \cref{sec:simresults} and \cref{sec:realresults} we construct random forests with $B$ trees in them but in the calculations above we assumed that the trees were all randomly weighted such that the random weights add up to an \textbf{expected value} of $B$, not always exactly equal to $B$. At the beginning of \cref{sec:moreresults} we have shown in detail that the difference between these two approaches are asymptotically negligible.

\item We will also consider the out-of-bag predictions in our implementation for calculating $\hat{F}^{(0)}(x)$ instead of the simple average (inbag prediction) of all the trees in the forest. This is also a form of assigning a weight to the trees in the forest (the weights aren't completely random but fixed given the dataset and the randomly selected subsets) but should also asymptotically give us the same results.

Using out-of-bag residuals could be thought of as akin to the honesty condition in \cite{wager2017estimation} for the second stage of the boosted forest since instead of using all the data for the residuals we use the data that was not used in construction of that particular tree. In fact, because of this we expect the out-of-bag approach to have more variability and hence the ratios $\frac{\overline{\hat{V}_{IJ}}}{V(\widehat{F})}$ in \cref{sec:simresults} should be higher than the expected value of 1. We shall also get slightly more conservative prediction intervals in \cref{sec:realresults} which will lead to higher coverage than the expected value of 95\%.

In \cref{sec:oobvsother} we compare our use of out-of-bag residuals with other boosting formulations where we find this version provides better predictive performance.
\end{itemize}

\subsection{Performance on Simulated Datasets}\label{sec:simresults}

Here we compare the performance of the Algorithms \ref{algo:boostforestsame} and \ref{algo:boostforestind} with different simulated datasets. The base learner which we will compare it against is just the simple random forest, i.e., without any boosting. We will also test the accuracy of our variance estimate by comparing it with the actual variance of the estimates.

Our model is $Y = \sum_{i=1}^5 X_i + \epsilon$, where $X \sim U([-1,1]^{15})$ and $\epsilon \sim N(0,1)$. We fix the points in $[-1,1]^{15}$ where we will make our predictions, given by
$$
p_1 = \mathbf{0}_{15},\; p_2 = (\frac13, \mathbf{0}_{14}^\top)^\top,\; p_3 = \frac1{3\sqrt{15}}*\mathbf{1}_{15},\; p_4 = 2p_3,\; p_5 = 3 p_3
$$
We chose $p_3, p_4$ and $p_5$ to have an idea of how distance of a test point from the the "center" of the dataset affects the performance of our algorithm. 

Out simulation runs for a 1000 iterations - in each of them we generate a dataset of size $n = 500$ and train a random forest and one-step boosted forests (both variants) with it with subsample size $k = 100$ and the number of trees $B$ in $(5000, 10000, 15000)$. For each of these settings we can find a prediction estimate at each of the $p_i$'s given by $\hat{F}_{i,j} = \hat{Y}_{i,j}$ and also corresponding variance estimates given by $\hat{V}_{i,j}$, for $i = 1,\dots, 5$, $j = 1,\dots, 1000$.

\afterpage{\clearpage}

\afterpage{\clearpage
\setlength{\tabcolsep}{4.5pt}
\begin{sidewaystable}[ht]
\centering
\begin{tabular}{|c|c||rrr||rrr||rrr||rrr||rrr|}
\hline
& & \multicolumn{3}{c||}{$p_1$} & \multicolumn{3}{c||}{$p_2$} & \multicolumn{3}{c||}{$p_3$} & \multicolumn{3}{c||}{$p_4$} & \multicolumn{3}{c|}{$p_5$} \\
\cline{3-5} \cline{6-8} \cline{9-11} \cline{12-14} \cline{15-17}\rule{0pt}{3ex}
  &   & RF & BFv1 & BFv2 & RF & BFv1 & BFv2 & RF & BFv1 & BFv2 & RF & BFv1 & BFv2 & RF & BFv1 & BFv2\\
\hhline{*{17}{=}}\rule{0pt}{3ex}

 & $\overline{\text{Bias}}$ & -0.0026 & -0.0050 & -0.0047 & -0.1062 & -0.0174 & -0.0172 & -0.1283 & -0.0181 & -0.0175 & -0.2645 & -0.0427 & -0.0421 & -0.4234 & -0.0910 & -0.0900\\ \rule{0pt}{3ex}
 & $\overline{\widehat{V}_{IJ}}$ & 0.0398 & 0.0890 & 0.0853 & 0.0397 & 0.0895 & 0.0856 & 0.0393 & 0.0883 & 0.0846 & 0.0382 & 0.0869 & 0.0834 & 0.0369 & 0.0856 & 0.0822\\ \rule{0pt}{3ex}
 & $\frac{\overline{\widehat{V}_{IJ}}}{V(\widehat{F})}$ & 1.6992 & 2.2342 & 2.1328 & 1.6834 & 2.2853 & 2.1757 & 1.6478 & 2.1984 & 2.1121 & 1.5366 & 2.0266 & 1.9443 & 1.4839 & 1.9106 & 1.8344\\ \rule{0pt}{3ex}
 & K.S. & 0.0686 & 0.1010 & 0.0943 & 0.2636 & 0.1225 & 0.1197 & 0.3002 & 0.1123 & 0.1029 & 0.5500 & 0.1456 & 0.1415 & 0.7760 & 0.1933 & 0.1905\\ \rule{0pt}{3ex}
 & C.C. & 98.8 & 99.7 & 99.5 & 95.8 & 99.3 & 99.1 & 94.8 & 99.9 & 99.9 & 76.0 & 99.1 & 99.0 & 38.4 & 98.3 & 98.1\\ \rule{0pt}{3ex}
\multirow{-6}{*}{\raggedleft\arraybackslash \rotatebox{90}{$B=5000$}} & P.I. & 0 & -70.10 & -70.74 & 0 & -13.12 & -13.71 & 0 & -0.40 & -0.05 & 0 & 52.88 & 52.94 & 0 & 74.01 & 74.10\\
\hline\rule{0pt}{3ex}

 & $\overline{\text{Bias}}$ & -0.0030 & -0.0056 & -0.0057 & -0.1065 & -0.0182 & -0.0180 & -0.1286 & -0.0180 & -0.0180 & -0.2651 & -0.0430 & -0.0433 & -0.4236 & -0.0902 & -0.0906\\ \rule{0pt}{3ex}
 & $\overline{\widehat{V}_{IJ}}$ & 0.0295 & 0.0709 & 0.0692 & 0.0294 & 0.0713 & 0.0695 & 0.0291 & 0.0702 & 0.0686 & 0.0281 & 0.0690 & 0.0671 & 0.0270 & 0.0678 & 0.0660\\ \rule{0pt}{3ex}
 & $\frac{\overline{\widehat{V}_{IJ}}}{V(\widehat{F})}$ & 1.2584 & 1.7856 & 1.7398 & 1.2500 & 1.8349 & 1.7839 & 1.2157 & 1.7457 & 1.7061 & 1.1407 & 1.6322 & 1.5827 & 1.0994 & 1.5316 & 1.4913\\ \rule{0pt}{3ex}
 & K.S. & 0.0328 & 0.0739 & 0.0734 & 0.2755 & 0.1031 & 0.0993 & 0.3159 & 0.0836 & 0.0825 & 0.5879 & 0.1216 & 0.1158 & 0.8133 & 0.1829 & 0.1829\\ \rule{0pt}{3ex}
 & C.C. & 96.9 & 98.8 & 98.7 & 92.5 & 98.6 & 98.4 & 90.5 & 99.1 & 98.9 & 65.0 & 98.1 & 98.1 & 25.3 & 97.1 & 97.0\\ \rule{0pt}{3ex}
\multirow{-6}{*}{\raggedleft\arraybackslash \rotatebox{90}{$B=10000$}} & P.I. & 0 & -69.29 & -69.44 & 0 & -12.26 & -12.57 & 0 & -0.15 & -0.07 & 0 & 53.53 & 53.33 & 0 & 74.32 & 74.29\\
\hline\rule{0pt}{3ex}

 & $\overline{\text{Bias}}$ & -0.0031 & -0.0054 & -0.0056 & -0.1065 & -0.0179 & -0.0178 & -0.1285 & -0.0176 & -0.0178 & -0.2649 & -0.0426 & -0.0425 & -0.4234 & -0.0897 & -0.0896\\ \rule{0pt}{3ex}
 & $\overline{\widehat{V}_{IJ}}$ & 0.0261 & 0.0649 & 0.0637 & 0.0260 & 0.0653 & 0.0641 & 0.0257 & 0.0643 & 0.0632 & 0.0247 & 0.0630 & 0.0618 & 0.0237 & 0.0618 & 0.0606\\ \rule{0pt}{3ex}
 & $\frac{\overline{\widehat{V}_{IJ}}}{V(\widehat{F})}$ & 1.1107 & 1.6293 & 1.5968 & 1.1123 & 1.6830 & 1.6482 & 1.0747 & 1.6061 & 1.5695 & 1.0074 & 1.4915 & 1.4611 & 0.9692 & 1.4014 & 1.3772\\ \rule{0pt}{3ex}
 & K.S. & 0.0336 & 0.0662 & 0.0614 & 0.2816 & 0.0950 & 0.0937 & 0.3228 & 0.0746 & 0.0730 & 0.6046 & 0.1099 & 0.1089 & 0.8305 & 0.1841 & 0.1830\\ \rule{0pt}{3ex}
 & C.C. & 96.0 & 98.2 & 98.1 & 90.7 & 98.2 & 98.2 & 87.1 & 98.9 & 98.6 & 59.5 & 97.4 & 97.4 & 22.2 & 96.2 & 96.2\\ \rule{0pt}{3ex}
\multirow{-6}{*}{\raggedleft\arraybackslash \rotatebox{90}{$B=15000$}} & P.I. & 0 & -69.77 & -70.12 & 0 & -12.73 & -12.96 & 0 & 0.22 & -0.33 & 0 & 53.51 & 53.48 & 0 & 74.41 & 74.48\\
\hline
\end{tabular}
\caption{\it Comparison of the two variants of the One-Step Boosted Forest (shorthands {\normalfont BFv1} and {\normalfont BFv2} for the two variants respectively) with respect to Random Forests (shorthands {\normalfont RF}). {\normalfont $\overline{\text{Bias}}$} is the average bias of the 3 methods \wrt the theoretical value of $F(p_i)$. $\frac{\overline{\widehat{V}_{IJ}}}{V(\widehat{F})}$ stands for the ratio of estimate of the variance to the variance of the estimates. We also calculate the Kolmogorov Smirnov statistic ({\normalfont K.S.}) to show how close we are to normality. {\normalfont C.C.} is the percentage coverage of a 95\% confidence interval. {\normalfont P.I.} stands for Performance Improvement defined in \eqref{eqn:perfimp}.}
\label{table:simresults}
\end{sidewaystable} }

We test the performance of our algorithm by the following metrics. The corresponding figures are in Table \ref{table:simresults}.

\begin{itemize}
\item The average bias is given by $\overline{\text{Bias}} = \frac1{1000} \sum_{j=1}^{1000} \hat{F}_{i,j} - F(p_i)$, where $F(x_1, \dots, x_{15}) = \sum_{i=1}^5 x_i$. We see that the bias is already fairly low at the origin and the boosted forest doesn't change that substantially. But as the target points moves away from the origin the improvement in bias becomes very obvious.

\item The variance estimate for each algorithm is given by $\overline{\widehat{V}_{IJ}} = \frac1{1000} \sum_{j=1}^{1000} \hat{V}_{i,j}$. For each $p_i$ the typical order for the variance estimate is BFv1 $>$ BFv2 $>$ RF but the value also decreases with $B$ as expected.

\item The ratio $\frac{\overline{\widehat{V}_{IJ}}}{V(\widehat{F})}$ shows the consistency of the infinitesimal Jackknife estimate. A value of 1 is ideal and we see that the empirical results aren't far away from 1. In fact the ratio decreases as $B$ gets larger as should be expected, see \cite{zhou2019asymptotic}.

\item K.S. gives us the Kolmogorov Smirnov statistics testing the hypothesis that the predictions should be normal with the mean given by the sample mean and the variability given by the variance estimate. Since smaller values of this statistic are better, thus we can see marked improvement for both variants of the Boosted Forest as compared to the base random forest. Improvements also get better as the target points move away from the origin and also as the number of trees increase - the second fact being consistent with the assumption $\frac{n}{B_n} \to 0$ mentioned in Theorem \ref{thm:normal}. Overall these numbers are fairly low and shows consistency with Theorem \ref{thm:normal} although the values are expected to be a bit high because we use variance estimate (consistent by \ref{thm:consistency}) instead of the (unknown) actual variance.

\item Constructing 95\% coverage intervals $\hat{F}_{i,j} \pm \Phi^{-1}(0.975) \sqrt{\hat{V}_{i,j}}$ we can check if $F(p_i)$ falls inside that interval for $j = 1, \dots, 1000$. C.C. denotes this coverage probability which we should expect to be close to 95\% for random forests and boosted forests. But we see that due to high bias values for random forests become worse as we move away from the origin. But boosted forests correct for the bias and thus the coverage is always at least 95\%. We also get more precision in our variance estimate as the number of trees increases and thus the coverage values also become less overinflated.

\item We also test \textit{Performance Improvement} (P.I.) which is defined as follows: Fixing $i \in [5]$ we obtain the estimated prediction MSE given by $MSE_{i,j} = (\hat{F}_{i,j} - F(p_i))^2$ for $j = 1, \dots, 1000$. Define
\begin{equation}
\text{Performance Improvement} (P.I.) = 1 - \frac{\sum_{j=1}^{1000} {MSE}_{i,j} \text{for BF}}{\sum_{j=1}^{1000} {MSE}_{i,j} \text{for RF}}. \label{eqn:perfimp}
\end{equation}
Since we are comparing against random forests their own P.I is 0. As for boosted forests P.I is actually worse for the points near the origin as boosting doesn't affect the bias too much but increases the variance quite a lot in comparison. But as we move further away from the origin our algorithm becomes effective at reducing bias compared to the increase in variance and thus we obtain significant improvements. This also gets better with an increase in number of trees since the variance estimates become more precise. Finally note that the two variants perform almost the same but later on in the paper we see Variant II performing better than Variant I.
\end{itemize}

We conclude that the Boosted Forest algorithms give better predictions than the usual random forest algorithm on simulated datasets and that Variant II is more powerful than Variant I.

In \cref{sec:simresultsnoisy} we have also done further simulations for datasets with different levels noise than the one in this section and observed that the one-step boosted forest algorithm is resilient to a moderate amount of noise but as expected, increasing noise reduces its performance. We have also looked into the performance of the method for a nonlinear signal in \cref{sec:simresultsnorm}. And finally note that we used out-of-bag residuals to construct the boosted forests in this section. We have compared it to boosted forests constructed with inbag residuals and also using bootstrapped subsamples in \cref{sec:oobvsother} and we have seen that the out-of-bag method is preferable.

\subsection{Performance on Real Datasets} \label{sec:realresults}

We applied the Boosted Forest algorithms (both variants) to 11 datasets in the UCI database (\cite{lichman2013uci}) which have a regression setting and compared its performance to the Gradient Boosting Machine algorithm [\cite{friedman2001greedy} and the \verb|R| package \verb|GBM|] and the Bias Correction algorithm in \cite{hooker2015bootstrap}. The results are reported in Table \ref{table:realresults}.

\afterpage{\clearpage
\SaveVerb{term}|GBM R|
\begin{table}[H]
\centering

\begin{tabular}{|l||rr||rr||rrr|}
  \hline\rule{0pt}{3ex}
 \multirow{3}{*}{Dataset} & \multicolumn{2}{c||}{Basics} & \multicolumn{2}{c||}{Improvement} & RF & BFv1 & BFv2 \\ \cline{2-8}\rule{0pt}{3ex}
 & n & d & GBM & BC & \multicolumn{3}{c|}{PI length} \\
 & k & varY & BFv1 & BFv2 & \multicolumn{3}{c|}{PI Coverage} \\
  \hline\hline\rule{0pt}{3ex}
  \multirow{2}{*}{Yacht-hydrodynamics} & 300 & 6 & 92.56 & 68.77 & 14.25 & 15.43 & 15.22 \\ \rule{0pt}{3ex}
  & 60 & 229.55 & 81.64 & 82.04 & 91.33 & 99.33 & 99.00 \\
  \hline\rule{0pt}{3ex}
  \multirow{2}{*}{BikeSharing-hour} & 17370 & 14 & 20.37 & 64.61 & 1.48 & 1.51 & 1.50 \\ \rule{0pt}{3ex}
  & 2000 & 2.21 & 73.64 & 73.99 & 100.00 & 100.00 & 100.00 \\
  \hline\rule{0pt}{3ex}
  \multirow{2}{*}{Concrete} & 1030 & 8 & 29.98 & 42.95 & 25.29 & 27.83 & 27.33 \\ \rule{0pt}{3ex}
  & 200 & 279.08 & 51.61 & 52.20 & 96.02 & 99.03 & 98.83 \\
  \hline\rule{0pt}{3ex}
  \multirow{2}{*}{Airfoil} & 1500 & 5 & -35.40 & 36.72 & 14.21 & 15.54 & 15.25 \\ \rule{0pt}{3ex}
  & 300 & 46.95 & 43.92 & 43.65 & 94.33 & 99.27 & 99.07 \\
  \hline\rule{0pt}{3ex}
  \multirow{2}{*}{Boston-housing} & 500 & 13.00 & -34.27 & 18.25 & 0.61 & 0.65 & 0.64 \\ \rule{0pt}{3ex}
  & 150 & 0.17 & 26.12 & 26.22 & 95.80 & 97.40 & 97.00 \\
  \hline\rule{0pt}{3ex}
  \multirow{2}{*}{Auto-mpg} & 390 & 7.00 & 7.00 & 16.13 & 11.40 & 11.94 & 11.86 \\ \rule{0pt}{3ex}
  & 50 & 61.03 & 21.21 & 20.79 & 93.85 & 96.41 & 95.90 \\
  \hline\rule{0pt}{3ex}
  \multirow{2}{*}{Wine-quality-white} & 4890 & 11.00 & -22.82 & 8.77 & 3.34 & 4.06 & 3.96 \\ \rule{0pt}{3ex}
  & 1000 & 0.79 & 11.57 & 11.42 & 98.45 & 99.51 & 99.45 \\
  \hline\rule{0pt}{3ex}
  \multirow{2}{*}{Parkinsons} & 5870 & 16.00 & -25.06 & 7.58 & 35.63 & 43.23 & 42.35\\ \rule{0pt}{3ex}
  & 1000 & 66.14 & 8.25 & 8.09 & 99.71 & 99.97 & 99.95 \\
  \hline\rule{0pt}{3ex}
  \multirow{2}{*}{Wine-quality-red} & 1590 & 11.00 & -9.12 & 5.42 & 2.59 & 2.90 & 2.85\\ \rule{0pt}{3ex}
  & 300 & 0.65 & 7.33 & 7.45 & 95.91 & 97.80 & 97.48 \\
  \hline\rule{0pt}{3ex}
  \multirow{2}{*}{SkillCraft} & 3330 & 18.00 & 2.30 & 2.06 & 4.37 & 5.05 & 5.00\\ \rule{0pt}{3ex}
  & 600 & 2.10 & 4.23 & 4.30 & 98.74 & 99.61 & 99.52 \\
  \hline\rule{0pt}{3ex}
  \multirow{2}{*}{Communities} & 1990 & 96.00 & -2.40 & 1.68 & 0.62 & 0.69 & 0.68\\ \rule{0pt}{3ex}
  & 400 & 0.05 & 2.93 & 3.05 & 96.88 & 98.19 & 98.09 \\
   \hline
\end{tabular}
\caption{\it Comparison of the two variants of the Boosted Forest with the GBM Algorithm [\cite{friedman2001greedy} and the \protect\UseVerb{term} package] and the Bias Correction Algorithm in \cite{hooker2015bootstrap}. We use shorthands for n = number of datapoints, d = number of features, k = size of subset (not applicable for GBM), PI = Prediction Interval \eqref{eqn:predint}, RF = Random Forests, GBM = Gradient Boosting Machine, BC = Bootstrap Bias Correction, BFv1 = Boosted Forest variant 1 (algorithm \ref{algo:boostforestsame}) and similarly for BFv2 (algorithm \ref{algo:boostforestind}). The Improvement and PI Coverage figures are in percentages.}
\label{table:realresults}
\end{table}}

For each dataset we use 10-fold cross-validation to calculate the prediction MSE and then record the improvement (in percentages) compared to the basic random forest algorithm. Improvement is simply 
$$
1-\frac{\text{prediction MSE for improved algorithm}}{\text{prediction MSE for random forest}}
$$
For the \verb|GBM| package in \verb|R| we used a 10\% validation set to select the optimal tuning parameter (number of trees/boosting steps) out of a maximum of 1000. We didn't use subsets but rather the full dataset to construct each tree in that ensemble. For random forests (\verb|randomForest| package in \verb|R|) and the two variants of our Boosted Forest algorithm we also used 1000 trees in the forests and randomly selected subsamples for each tree the size of which is given by the number k in Table \ref{table:realresults}.

We can see that the GBM algorithm doesn't always have improvements over random forests and hence is not reliable as a good reference. Further the Boosted Forest algorithm has consistently registered greater improvement compared to the the Bias Correction algorithm (\cite{hooker2015bootstrap}). Variant 2 of our algorithm slightly outperforms variant 1 in most  cases.

We further validate our variance estimate by constructing test set confidence intervals. A 95\% prediction interval for the datapoint $Z_i = (Y_i,X_i)$ is given by
\begin{equation}
\left( \hat{Y}_i - \Phi^{-1}(0.975)\sqrt{\hat{V}_i + \hat{V}_e},\hat{Y}_i + \Phi^{-1}(0.975)\sqrt{\hat{V}_i + \hat{V}_e} \right) \label{eqn:predint}
\end{equation}
where $\hat{Y}_i$ is the estimate, $\hat{V}_i$ is the variance estimate and $\hat{V}_e = \frac1n \sum_{i=1}^n (\hat{Y}_i - Y_i)^2$ is the residual MSE.

We see that when comparing Boosted Forests (for both variants) with the basic random forest algorithm the length of the prediction interval increases slightly but the prediction coverage (in percentages) increases significantly. The increment in the length of the prediction interval can be attributed to the increase in variability due to boosting. The same can also partially explain the increase in prediction coverage but the main reason for this increase is the reduction in bias due to boosting which leads to better ``centering'' of the prediction interval.

For these datasets as well we used the out-of-bag residuals to construct the random forest in the boosting step. A comparison with other approaches is in \cref{sec:oobvsother}.

\section{Conclusion}

Our algorithm, the \textbf{One-Step Boosted Forest} fits a random forest on a given dataset and then fits another one on the residuals of the former. The sum of these two random forests give us our estimate. This a boosting method for random forests which, even if applied only once, provides performance improvements compared to base algorithm. Since it is a boosting method on a previously bagged estimate, the result should be a very stable algorithm.

The boosted forest also provides an estimate of its own variance which can be obtained with nothing more than the computation needed to calculate the boosted forest estimate itself. We have shown that our method leads to substantial reductions in bias (compared to a small increment in variance) in the regression setting and thus the predictive mean squared error. More such boosting steps can be chained to get more improvements and we devised a fairly simple criteria for when to stop such further boosting. We have only tested our method against the random forest and gradient boosting algorithms but we expect similar results for other ensemble methods.

Following the discussion in \cref{sec:extensions} we could suggest some potential stopping rules for the boosting iteration. As in the original boosting framework in \cite{friedman2001greedy}, we expect that while boosting reduces the bias associated with random forests, it will incur greater variance as boosting progresses and a stopping rule can be based on test set error. Here we can also make use of theoretical results for random forests where we observe that in the $M$-step boosted forest $\left(\hat{F}^{(m)}(x) \left| \hat{F}^{(0)}(x), \dots, \hat{F}^{(m-1)}(x) \right.\right)$ will have an asymptotic normal distribution. We can thus test whether the expectation of the last step is significantly different from zero -- i.e., did the last step contribute to bias reduction?  Tests of this form can be constructed by using a collection of tests points $(x_1,\ldots,x_q)$ for which $\hat{F}^{(m)}(x_1),\ldots,\hat{F}^{(m)}(x_q)$ has a multivariate normal distribution which can be used to for a $\chi^2$ test; similar approaches to testing structure in random forests were described in \cite{mentch2016quantifying}, \cite{mentch2017formal} and \cite{Zhou2Hooker2018}. The development of stopping criteria tests is left to future work. 


We note that commonly-employed diagnostic tools for random forests remain applicable to their boosted version. In particular, variable importance measures based on split improvement scores (e.g. \cite{friedman2001greedy}) can be summed across the forests to arrive at combined importance scores.  Permutation-based methods described in \cite{breiman2001random} should be applied to the combined model.  We note that both types of variable importance are potentially misleading; see \cite{strobl2007bias, ishwaran2007variable, hooker2007generalized, strobl2008conditional, hooker2019please} and therefore do not recommend their use, although \cite{zhou2019unbiased} produces a potential fix to split improvement scores.

\nocite{*} 
\printbibliography

\newpage
\appendix

\section{Exploring the regularity condition}

Our results in \cref{sec:theory} rest on the validity of Condition \ref{cond:regularity}, essentially stating that the variability of $\hat{F}^{(0)}$ has a negligible effect on $\hat{F}^{(1)}$. Although we find that this assumption appears to hold empirically, its asymptotic validity will likely depend strongly on the specific details of how random forest trees are built. First we will give some mathematical intuition for when we can expect this to be the case and then we show empirically that our assumption holds true in simple cases.

\subsection{Kernel Analogy for the regularity condition}\label{sec:kernel}

One way to examine $\hat{F}^{(1)}(x) - \tilde{F}^{(1)}(x)$ is to consider the leaf $L$ of tree $j$ within which $x$ falls and for which the difference in predictions is
$$
D = \frac{1}{|L|} \sum_{i \in L} \left(\hat{F}^{(0)}(X_i) - \EE\left[\hat{F}^{(0)}(X_i)\right]\right)
$$
the average deviation of $\hat{F}^{(0)}(X)$ from its expectation. So long as correlation between predictions within each leaf decays fast enough -- equivalent to the covariance matrix having a finite sum -- and so long as these differences do not change the structure of the tree when $n$ is sufficiently large, then $D$ should decrease relative to $Y_i - \hat{F}^{(0)}(X_i)$ and Condition \ref{cond:regularity} ought to hold.

We might expect low correlation among residuals when the trees in $\hat{F}^{(1)}$ have very different leaves from those in $\hat{F}^{(0)}$; either in choosing different covariates, or in being smaller and picking out more detail. These are exactly the conditions under which we expect one-step boosting to have an advantage: when $\hat{F}^{(1)}$ targets substantially different structure compared to $\hat{F}^{(0)}$.

However, the specific conditions required for this to occur are difficult to verify. An alternative is to use the connections to kernel methods developed by \cite{scornet2016random}. The kernels derived there are not given with respect to an explicit bandwidth but, loosely speaking, smaller bandwidths correspond to deeper trees.  Here we show that the equivalent condition for boosted kernel estimates holds if the second stage estimate either uses different covariates to the first stage, or has a smaller bandwidth.

Suppose we have a kernel $K$, a dataset $\left(X_i,Z_i,Y_i\right)_{i=1}^n$ and bandwidths $h_1$ and $h_2$. For the model $Y_i = f_0(X_i) + f_1(Z_i) + \epsilon_i$, $\epsilon_i \overset{iid}{\sim} (0,\sigma^2)$ define the following kernel estimators.
\begin{align*}
\hat{f}_0(x) &= \frac1{nh_1} \sum_{i=1}^n Y_i K\left(\frac{x-X_i}{h_1}\right) \\
\hat{f}_1(x) &= \frac1{nh_2} \sum_{i=1}^n \left(Y_i - \hat{f}_0(X_i)\right) \cdot K\left(\frac{z-Z_i}{h_2}\right) \\
\tilde{f}_1(x) &= \frac1{nh_2} \sum_{i=1}^n \left(Y_i - \EE\left[\hat{f}_0(X_i)\right]\right) \cdot K\left(\frac{z-Z_i}{h_2}\right)
\end{align*}

For completely randomised forests the kernels do not depend on the response. But in the standard random forest algorithm, the splits we make for each parent-children combination depends on the response. As a result the dataset used in the kernel-type weights given to each response $Y_i$ in the two stages of the boosted forest is expected to differ. This is the motivation behind the definition, more specifically using $X$ and $Z$ instead of only $X$. So if we consider the analogous case of the definition of $\hat{f}_1$ ($\hat{F}^{(1)}$) and $\tilde{f}_1$ ($\tilde{\hat{F}}^{(1)}$), then the relationship between $Z$ and $X$ falls between the two extremes
\begin{enumerate}[(a)]
\item $Z$ is the same as $X$, i.e., it's joint distribution is the same as $\ind_{\{X=Z\}} \times g(X)$, where $g$ is some density. Loosely speaking this is the same as saying that the joint distribution is concentrated on the "diagonal".
\item $Z$ and $X$ has a joint distribution, which is different to the one above, and loosely speaking it isn't concentrated on the "diagonal".
\end{enumerate}
We will show that
\begin{equation}
\frac{\hat{f}_1(x) - \tilde{f}_1(x)}{\sqrt{Var(\hat{f}_1(x))}} \xrightarrow{p} 0, \label{eqn:kernelregularity}
\end{equation}
holds for both cases (a) and (b) above.

For case (a) assume that $h_2/h_1 \to 0$ as $n \to \infty$
\begin{align*}
\hat{f}_1(x) - \tilde{f}_1(x) &= -\frac1{nh_2} \sum_{i=1}^n \left(\hat{f}_0(X_i) - \EE[\hat{f}_0(X_i)]\right) \cdot K\left(\frac{x-X_i}{h_2}\right) \\
&= -\frac1{n^2h_1h_2} \sum_{i=1}^n \sum_{j=1}^n \left(Y_j - \EE[Y_j]\right) \cdot K\left(\frac{X_i-X_j}{h_1}\right) \cdot K\left(\frac{x-X_i}{h_2}\right) \\
&= -\frac1n \sum_{j=1}^n \epsilon_j \left[\frac1{h_1h_2} \cdot \frac1n \sum_{i=1}^n K\left(\frac{X_i-X_j}{h_1}\right) \cdot K\left(\frac{x-X_i}{h_2}\right) \right] \\
&\approx -\frac1n \sum_{j=1}^n \epsilon_j \left[\frac1{h_1h_2} \cdot \int K\left(\frac{u-X_j}{h_1}\right) \cdot K\left(\frac{x-u}{h_2}\right) g(u) du \right] \\
&= -\frac1n \sum_{j=1}^n \epsilon_j \cdot \frac1{h_1} \int K\left( \frac{h_2}{h_1}t + \frac{x-X_j}{h_1} \right) K(t) g(x+th_2) dt \\
&\approx -\frac1n \sum_{j=1}^n \epsilon_j \cdot \frac1{h_1} K\left(\frac{x-X_j}{h_1}\right) \int K(t)g(x+th_2)dt
\end{align*}
Hence $Var(\hat{f}_1(x) - \tilde{f}_1(x)) = O(\frac1{nh_1})$. But $Var(\hat{f}_1(x)) = O(\frac1{nh_2})$. So we have
$$
\frac{Var(\hat{f}_1(x) - \tilde{f}_1(x))}{Var(\hat{f}_1(x))} = O\left(\frac{1/nh_1}{1/nh_2}\right) = O\left(\frac{h_2}{h_1}\right) \to 0
$$
Thus $\frac{\hat{f}_1(x) - \tilde{f}_1(x)}{\sqrt{Var(\hat{f}_1(x))}}$ converges to 0 in $\LLL^2$ and hence in probability as well.

We saw that for the case $Z=X$ \eqref{eqn:kernelregularity} holds under the condition that the bandwidth for the second stage is smaller than that of the first stage. In terms of random forests (following the calculations in \cite{scornet2016random}), the equivalent quantity to bandwidth is the inverse of the depth of each tree in the forest. So for construction of the boosted forest if we build the second stage trees to be deeper than the first stage trees then we should expect \ref{cond:regularity} to hold. This also makes intuitive sense, since we expect the first stage random forest to pick up most of the signal and to pick up any additional signal leftover in the residuals we would need to build deeper forests in the second stage.

For case (b) we only need to assume that $h_2 \to 0$ as $n \to \infty$ as per the following calculations
\begin{align*}
\hat{f}_1(z) - \tilde{f}_1(z) &= -\frac1{nh_2} \sum_{i=1}^n \left(\hat{f}_0(X_i) - \EE[\hat{f}_0(X_i)]\right) \cdot K\left(\frac{z-Z_i}{h_2}\right) \\
&= -\frac1{n^2h_1h_2} \sum_{i=1}^n \sum_{j=1}^n \left(Y_j - \EE[Y_j]\right) \cdot K\left(\frac{X_i-X_j}{h_1}\right) \cdot K\left(\frac{z-Z_i}{h_2}\right) \\
&= -\frac1n \sum_{j=1}^n \epsilon_j \left[\frac1{h_1h_2} \cdot \frac1n \sum_{i=1}^n K\left(\frac{X_i-X_j}{h_1}\right) \cdot K\left(\frac{z-Z_i}{h_2}\right) \right] \\
&\approx -\frac1n \sum_{j=1}^n \epsilon_j \left[\frac1{h_1h_2} \cdot \int K\left(\frac{u-X_j}{h_1}\right) \cdot K\left(\frac{z-v}{h_2}\right) g(u,v) du dv \right] \\
&= -\frac1n \sum_{j=1}^n \epsilon_j \cdot \int K(s) K(t) g(X_j+sh_1, z+th_2) dt
\end{align*}
Hence in this case $Var(\hat{f}_1(x) - \tilde{f}_1(x)) = O(\frac1n)$. But $Var(\hat{f}_1(x)) = O(\frac1{nh_2})$. So we have
$$
\frac{Var(\hat{f}_1(x) - \tilde{f}_1(x))}{Var(\hat{f}_1(x))} = O\left(\frac{1/n}{1/nh_2}\right) = O(h_2) \to 0,
$$
since we expect the bandwidth ($h_2$) for a kernel to be narrower as the amount of data ($n$) increases. Thus $\frac{\hat{f}_1(x) - \tilde{f}_1(x)}{\sqrt{Var(\hat{f}_1(x))}}$ converges to 0 in $\LLL^2$ and hence in probability as well.

So if $Z$ and $X$ are not the same then \eqref{eqn:kernelregularity} holds for any bandwidths $h_1$ and $h_2$. This is an idealised scenario that might not hold in the standard boosted forest construction and in particular it assumes the  $\hat{F}^{(0)}(x) - \EE[\hat{F}^{(0)}(x)]$ has negligible impact on the structure of the trees in $\hat{F}^{(1)}$. Nonetheless, we believe that our discussion does provide some intuition for when Condition \ref{cond:regularity} is likely to hold and emphasize that it does appear to be reasonable in practice.

\subsection{Empirical evaluation of the regularity condition}\label{sec:regularplot}

We consider a small scale empirical excursion into the regularity condition. Our model is $Y = X_1 + X_2 + \epsilon$, where $X \sim U([-1,1]^{5})$ and $\epsilon \sim N(0,0.25)$. First we fix a set of 1000 test points used for the final predictions of $\hat{F}^{(1)}(x)$ and $\check{F}^{(1)}(x)$. For calculating the values of $\EE \hat{F}^{(0)}(X_i)$ we fit 40 random forests on different training datasets and then take the average prediction from those on the fixed test points. We vary the sample size ($n$) of the training dataset in the range (100, 500, 1000, 2000) and to fix the maximum depth of the trees we set the nodesizes to be in the scale of $\log(n)$. The we measured the variance in the regularity condition by computing $\hat{F}^{(1)}$ and $\check{F}^{(1)}$ on the same 40 data sets used to estimate $\EE \hat{F}^{(0)}$. The ratio decreases as $n$ increases as shown in the figure below (with 1 standard deviation error-bars).

\begin{figure}[H]
\centering
\includegraphics[width=.5\textheight, height=.5\columnwidth]{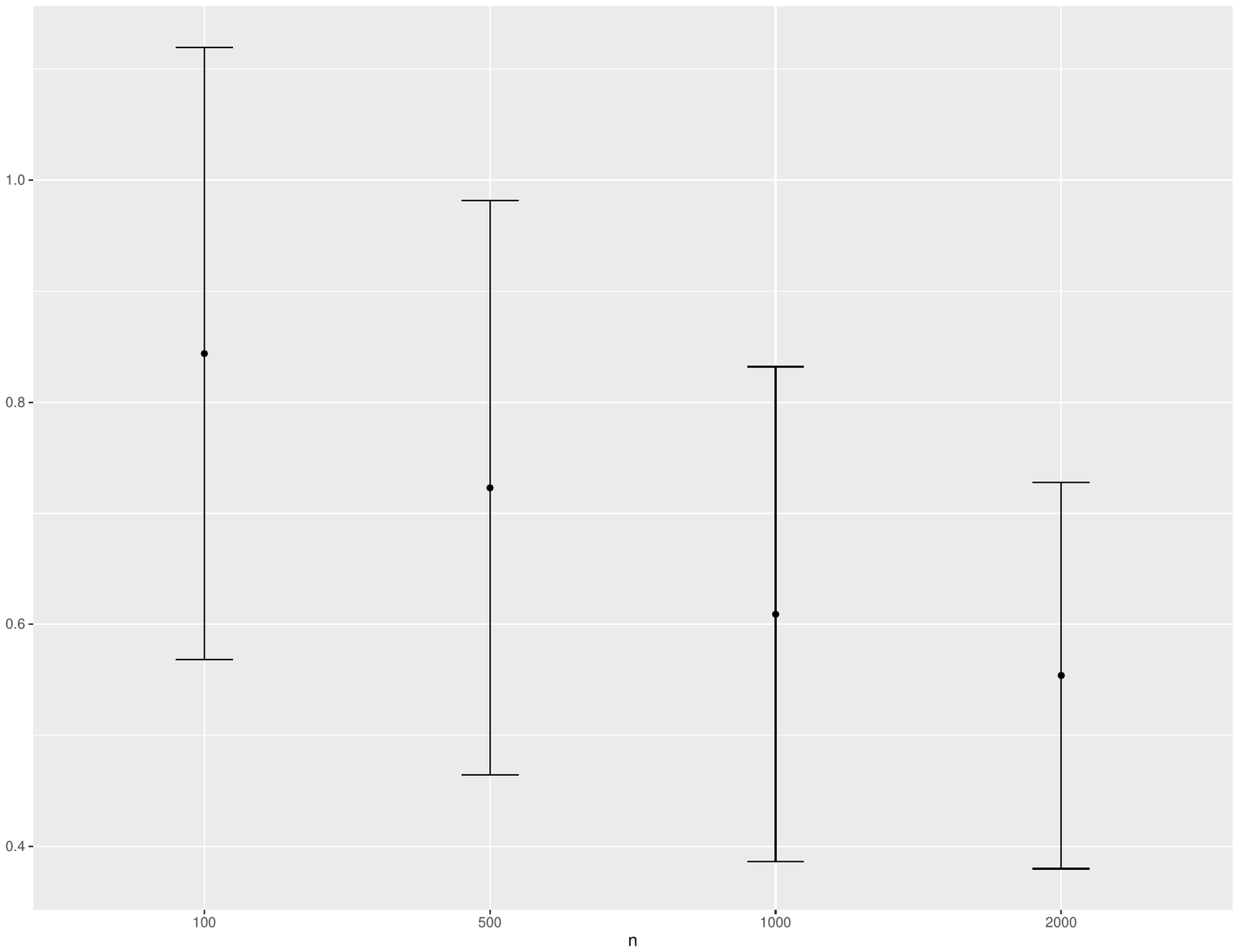}
\end{figure}

\section{Proofs}\label{sec:proofs}

\subsection{U-statistics and theoretical variance of one-step Boosted Forests}\label{sec:Uvarest}

To find the theoretical variance of the \textit{one-step boosted forest} we must familiarise ourselves with a result about U-statistics from \cite{hoeffding1948class}. [See \cite{lee1990u} for more details.]

\begin{lem}\label{lem:ustatvar}
If $h(z_1, \dots, z_k)$ is a symmetric function and a consistent estimator of $\theta$, then the U-statistic $U$ with kernel $h$ defined by
$$
U(Z_1, \dots, Z_n) = \frac1{{n \choose k}} \sum_{I \subseteq [n]: |I| = k} h(Z_I) = \frac1{{n \choose k}} \sum_{I \subseteq [n]: |I| = k} h(Z_{I_1}, \dots, Z_{I_k})
$$
is the MVUE for $\theta$ given an i.i.d. dataset $Z_{[n]} = (Z_1, \dots, Z_n)$. Further the U-statistic is asymptotically normal with variance $\frac{k^2}{n} \zeta_{1,k}$, where
\begin{align*}
\zeta_{c,k} &= cov(h(Z_1,\dots,Z_c,Z_{c+1},\dots,Z_k), h(Z_1,\dots,Z_c,Z'_{c+1},\dots,Z'_k)) \\
&= var(\EE(h(Z_1,\dots,Z_k) | Z_1 = z_1, \dots, Z_c = z_c))
\end{align*}
\end{lem}

\textbf{Note:} $\zeta_{k,k} = var(h(Z_1,\dots,Z_k))$ follows from the above definition of $\zeta_{c,k}$.

To calculate the theoretical prediction variance of one-step boosted forests we will first condition over the weights to get a complete U-statistic and then we use the above lemma. As a reminder we had stated before that $\zeta_{k,k} = var(h(Z_1,\dots,Z_k))$ and had assumed, without loss of generality, that $\EE [T(x; Z_I)] = 0$ for all $I \subseteq [n]: |I| = k$.

\textbf{Variant I}: For the first variant of boosted forests $w_I^{(1)} = w_I^{(0)}$. Then we have
\begin{align*}
V_{same}(x) := var(\hat{F}(x)) &= var_Z(\EE_w \hat{F}(x)) + \EE_Z(var_w \hat{F}(x))\\
&= var_Z \left[\EE_w \left(\frac1{{n \choose k}} \sum_{I \subseteq [n]: |I| = k} w_I^{(0)} \big(T(x; Z_{I}^{(0)}) + T(x; Z_{I}^{(1)}) \big)\right)\right]\\
&\qquad\qquad + \EE_Z \left[var_w \left(\frac1{{n \choose k}} \sum_{I \subseteq [n]: |I| = k} w_I^{(0)} \big(T(x; Z_{I}^{(0)}) + T(x; Z_{I}^{(1)}) \big)\right)\right] \\
&= var_Z\left[\frac1{{n \choose k}} \sum_{I \subseteq [n]: |I| = k} \left(T(x; Z_{I}^{(0)}) + T(x; Z_{I}^{(1)}) \right) \right] \\
&\qquad\qquad + \EE_Z\left[\frac{c}{{n \choose k}^2} \sum_{I \subseteq [n]: |I| = k} \left(T(x; Z_{I}^{(0)}) + T(x; Z_{I}^{(1)}) \right)^2 \right] \\
&= \frac{k^2}{n}\zeta_{1,k} \cdot (1 + \epsilon_{k,n}) + \frac{c}{{n \choose k}} \zeta_{k,k},
\end{align*}
where under the conditions of Lemma \ref{lem:ustatvar}, $\epsilon_{k,n} \to 0$ as $n \to \infty$ for a fixed $k$. But also under Condition \ref{cond:mainLF} $\epsilon_{k,n} \to 0$ if $k = o(n)$. The $\zeta$ values are based on the kernel being $T(x; Z_{I}^{(0)}) + T(x; Z_{I}^{(1)})$, i.e., sum of the trees in the two stages rather than the individual trees for separate stages.

\textbf{Variant II:} Here $w_I^{(0)}$ and $w_I^{(1)}$ are independent sets of binary weights. Then we see that
\begin{align*}
V_{ind}(x) := var(\hat{F}(x)) &= var_Z(\EE_w \hat{F}(x)) + \EE_Z(var_w \hat{F}(x)) \\
&= var_Z \left[\EE_w \left(\frac1{{n \choose k}} \sum_{I \subseteq [n]: |I| = k} \big( w_I^{(0)} T(x; Z_{I}^{(0)}) + w_I^{(1)} T(x; Z_{I}^{(1)}) \big)\right)\right] \\
&\qquad\qquad + \EE_Z \left[var_w \left(\frac1{{n \choose k}} \sum_{I \subseteq [n]: |I| = k} \big( w_I^{(0)} T(x; Z_{I}^{(0)}) + w_I^{(1)} T(x; Z_{I}^{(1)}) \big)\right)\right] \\
&= var_Z\left[\frac1{{n \choose k}} \sum_{I \subseteq [n]: |I| = k} \left(T(x; Z_{I}^{(0)}) + T(x; Z_{I}^{(1)}) \right) \right] \\
&\qquad\qquad + \EE_Z\left[\frac{c}{{n \choose k}^2} \sum_{I \subseteq [n]: |I| = k} \left(T(x; Z_{I}^{(0)})^2 + T(x; Z_{I}^{(1)})^2 \right) \right] \\
&= \frac{k^2}{n}(\zeta^{(0)}_{1,k} + \zeta^{(1)}_{1,k} + 2\zeta^{(0,1)}_{1,k}) \cdot (1 + \epsilon_{k,n}) + \frac{c}{{n \choose k}} (\zeta^{(0)}_{k,k} + \zeta^{(1)}_{k,k}), 
\end{align*}
where $\epsilon_{k,n} \to 0$ under the same conditions described in the previous paragraph. Here the $\zeta^{(0)}$ and $\zeta^{(1)}$ values are based on the kernels being $T(x; Z_{I}^{(0)})$ and $T(x; Z_{I}^{(1)})$ respectively. Also here $\zeta^{(0,1)}_{1,k}$ is the (scaled) covariance term between the trees in the two stages, $T(x; Z_{I}^{(0)})$ and $T(x; Z_{I}^{(1)})$.

Note that the first of the two terms in \eqref{eqn:bfsamevar} and \eqref{eqn:bfindvar} correspond to $var_Z(\EE_w \hat{F}(x))$, i.e., variance of a complete (unweighted) U-statistic. So then theoretically $\zeta_{1,k} = \zeta^{(0)}_{1,k} + \zeta^{(1)}_{1,k} + 2\zeta^{(0,1)}_{1,k}$. This property will be useful in the proof of Theorem \ref{thm:normal}. The different formulae for the same quantity become useful for estimation purposes. We can use the structure in the construction of variant I (i.e., same subsets on both stages) to estimate $\zeta_{1,k}$ but due to the absence of such structure in variant II we shall need to estimate each term in $\zeta^{(0)}_{1,k} + \zeta^{(1)}_{1,k} + 2\zeta^{(0,1)}_{1,k}$ separately. 

\subsection{Asymptotic Normality of One-Step Boosted Forests} \label{sec:proofmain}

\textit{Proof of Theorem \ref{thm:normal}}. We will follow the same notation as in \cref{sec:basicalgo} and keep in mind the simplification in notation discussed at the end of \cref{sec:assumption}. We first focus on the case of Variant I and or simplicity of notation we define $M_n = {n \choose k_n}$ and $T_I = T(x;Z_I^{(0)})+T(x;Z_I^{(1)})$. Then we have that

$$
\hat{F}(x) = \frac1{M_n} \sum_{(I)} w_I^{(0)}T_I,
$$
where the sum $(I)$ is over all subsets $I \subseteq [n]; |I| = k$. Now we define $\tilde{F}(x) = \frac1{M_n} \sum_{(I)} T_I$ and $w'_I = w_I^{(0)}-1$ to be independent binary random variables with the following properties
\begin{align*}
\EE[w'_I] &= \EE[w_I^{(0)}] - 1 = 0; \\
\EE[(w'_I)^2] &= \EE[(w_I^{(0)})^2] - 2\EE[w_I^{(0)}] + 1 = M_n/B_n - 1; \\
\text{for}\; I \neq J, \;\EE[w'_Iw'_J] &= \EE[w_I^{(0)}w_J^{(0)}] - \EE[w_I^{(0)}] - \EE[w_J^{(0)}] + 1 = 1 \cdot 1 - 1 -1 + 1 = 0
\end{align*}

Also note that $\EE[T^2] = \zeta_{k_n,k_n}$ and hence
\begin{align*}
\EE[(\hat{F}(x) - \tilde{F}(x))^2] &= \EE\left[ \left(\frac1{M_n} \sum_{(I)} (w_I^{(0)}T_I - T_I)\right)^2 \right] = \EE\left[ \left(\frac1{M_n} \sum_{(I)} w'_I T_I\right)^2 \right] \\
&= \frac1{M_n^2} \left[ \EE\left[\sum_{(I)} (w'_I T_I)^2 \right] + \EE\left[ \sum_{I \neq J} w'_Iw'_J T_IT_J \right] \right] \\
&= \frac1{M_n^2} \left[ M_n\left(\frac{M_n}{B_n} - 1\right)\EE[T^2] + 2{M_n \choose 2} \cdot 0 \right] \qquad [\text{we condition over } T] \\
&= \left(\frac1{B_n} - \frac1{M_n}\right) \EE[T^2] \\
\implies \frac{\EE[(\hat{F}(x) - \tilde{F}(x))^2]}{\frac{k_n^2}{n}\zeta_{1,k_n} + \frac1{B_n} \zeta_{k_n,k_n}} &= \frac{\left(\frac1{B_n} - \frac1{M_n}\right) \zeta_{k_n,k_n}}{\frac{k_n^2}{n}\zeta_{1,k_n} + \frac1{B_n} \zeta_{k_n,k_n}} = \frac{1 - \frac{B_n}{M_n}}{\frac{B_n}{n} \cdot k_n \cdot \frac{k_n\zeta_{1,k_n}}{\zeta_{k_n,k_n}} + 1} \to 0 \\
&\qquad\left[\because k_n \to \infty, \frac{B_n}{n} \to \infty, \lim\limits_{n \to \infty} \frac{k_n\zeta_{1,k_n}}{\zeta_{k_n,k_n}} \neq 0 \text{ and } B_n \leq M_n \right]
\end{align*}
Thus similar to results in \cite{scornet2016asymptotics} we showed that $\hat{F}(x)$ and $\tilde{F}(x)$ are asymptotically close (with appropriate scaling). We also need to show that their variances are asymptotically close as well.
\begin{align*}
\frac{\frac{k_n^2}{n}\zeta_{1,k_n}}{\frac{k_n^2}{n}\zeta_{1,k_n} + \frac1{B_n} \zeta_{k_n,k_n}} &= \frac1{1 + \frac{n}{B_n} \cdot \frac{1}{k_n} \cdot \frac{\zeta_{k_n,k_n}}{k_n\zeta_{1,k_n}}} \to 1
\end{align*}
Now we know that $\tilde{F}(x)$ is a complete U-statistic and hence by Theorem 1(i) of \cite{mentch2016quantifying} \footnote{Note that the assumptions concerning $k_n$ and $\zeta_{1,k_n}$ in Theorem 1 of \cite{mentch2016quantifying} are not consistent. We assume $k_n/n \to 0$ and $\lim \frac{k_n\zeta_{1,k_n}}{\zeta_{k_n,k_n}} \neq 0$ instead and if we follow the proof of main theorem of that paper then these conditions can easily be shown to give us the same result.} and Theorems 11.2 and 12.3 of \cite{van2000asymptotic} we have
\begin{align*}
\frac{\tilde{F}(x)}{\left(\frac{k_n^2}{n}\zeta_{1,k_n}\right)^{1/2}} &\xrightarrow{\DDD} Z \sim \NNN(0,1) \\
\implies \frac{\hat{F}(x)}{\left(\frac{k_n^2}{n}\zeta_{1,k_n} + \frac1{B_n} \zeta_{k_n,k_n}\right)^{1/2}} &= \frac{\tilde{F}(x)}{\left(\frac{k_n^2}{n}\zeta_{1,k_n}\right)^{1/2}} \cdot \left(\frac{\frac{k_n^2}{n}\zeta_{1,k_n}}{\frac{k_n^2}{n}\zeta_{1,k_n} + \frac1{B_n} \zeta_{k_n,k_n}}\right)^{1/2} + \frac{\hat{F}(x) - \tilde{F}(x)}{\left(\frac{k_n^2}{n}\zeta_{1,k_n} + \frac1{B_n} \zeta_{k_n,k_n}\right)^{1/2}} \\
&\xrightarrow{\DDD} Z \cdot 1^{1/2} + 0 \qquad[\because \text{ convergence in } \LLL^2 \text{ implies convergence in probability}]\\
&= Z \sim \NNN(0,1) \qquad[\text{by Slutsky's Theorem}]
\end{align*}
So if we define $\sigma_n^2(x) = \frac{k_n^2}{n}\zeta_{1,k_n} + \frac1{B_n} \zeta_{k_n,k_n}$ then we have central limit theorem for Variant I of the One-Step Boosted Forest (Algorithm \ref{algo:boostforestsame}).

Similarly when we define $\hat{F}(x)$ to be Variant II of the One-Step Boosted Forest then we can easily see that $\tilde{F}(x) = \tilde{F}^{(0)}(x) + \tilde{F}^{(1)}(x)$ satisfies [$\because \zeta_{1,k_n} = \zeta^{(0)}_{1,k_n} + \zeta^{(1)}_{1,k_n} + 2\zeta^{(0,1)}_{1,k_n}$]
$$
\frac{\tilde{F}(x)}{\left(\frac{k^2}{n}(\zeta^{(0)}_{1,k_n} + \zeta^{(1)}_{1,k_n} + 2\zeta^{(0,1)}_{1,k_n})\right)^{1/2}} \xrightarrow{\DDD} \NNN(0,1)
$$
Following similar calculations as above we can also easily show that
$$
\frac{\frac{k^2}{n}(\zeta^{(0)}_{1,k_n} + \zeta^{(1)}_{1,k_n} + 2\zeta^{(0,1)}_{1,k_n})}{\frac{k^2}{n}(\zeta^{(0)}_{1,k_n} + \zeta^{(1)}_{1,k_n} + 2\zeta^{(0,1)}_{1,k_n}) + \frac1{B_n} (\zeta^{(0)}_{k_n,k_n} + \zeta^{(1)}_{k_n,k_n})} \to 1
$$
and that $\hat{F}(x)$ and $\tilde{F}(x)$ are close together (with appropriate scaling). Thus if we define $\sigma_n^2(x) = \frac{k^2}{n}(\zeta^{(0)}_{1,k_n} + \zeta^{(1)}_{1,k_n} + 2\zeta^{(0,1)}_{1,k_n}) + \frac1{B_n} (\zeta^{(0)}_{k_n,k_n} + \zeta^{(1)}_{k_n,k_n})$ we also have proved central limit theorem for Variant II of the One-Step Boosted Forest (Algorithm \ref{algo:boostforestind}). \qed

\subsection{Joint Normality of the Boosting Steps}\label{sec:proofnormal}

This theorem is a generalisation of Theorem \ref{thm:normal} stating that the boosting steps have a joint asymptotic normal distribution (with appropriate scaling). For this we need to assume a slightly more stringent assumption for the variances of each step [$\zeta_{1,k_n}^{(0)}$ and $\zeta_{1,k_n}^{(1)}$] rather than the variance of their sum as we did for the previous theorem.

\begin{thm}\label{thm:main}
Assume that the dataset $Z_1^{(0)}, Z_2^{(0)}, \dots \overset{iid}{\sim} D_{Z^{(0)}}$ and $\tilde{F}^{(0)}$ and $\tilde{F}^{(1)}$ are first and second stage random forests constructed based the tree kernel $T$ with mean zero and on all possible subsets of size $k_n$. Also assume that $\EE T^4(x; Z_1,\dots,Z_{k_n}) \leq C < \infty$ for all $x, n$ and some constant $C$ and that $\lim\limits_{n \to \infty} \frac{\zeta_{1,k_n}^{(0)}}{\zeta_{1,k_n}^{(1)}} \in (0, \infty)$. Then under conditions \ref{cond:regularity} and \ref{cond:mainLF} and as long as $k_n, B_n \to \infty$ such that $\frac{k_n}{n} \to 0$ and $\frac{n}{B_n} \to 0$ as $n \to \infty$ as well as $\lim\limits_{n \to \infty} \frac{k_n \zeta_{1,k_n}^{(0)}}{\zeta_{k_n,k_n}^{(0)}} \neq 0$ and $\lim\limits_{n \to \infty} \frac{k_n \zeta_{1,k_n}^{(1)}}{\zeta_{k_n,k_n}^{(1)}} \neq 0$ we have
\begin{equation}
\Sigma_n^{-1/2} \begin{pmatrix} \tilde{F}^{(0)}(x) \\ \tilde{F}^{(1)}(x) \end{pmatrix} \xrightarrow{\DDD} \NNN(\mathbf{0}, \mathbf{I}_2), \label{eqn:conv}
\end{equation}
for some $2 \times 2$ covariance matrix $\Sigma_n$, which can be consistently estimated by the infinitesimal Jackknife covariance matrix estimator given by
\begin{equation}
\hat{\Sigma}_{IJ} = \sum_{i=1}^n \begin{pmatrix}
cov_*[N_{i,b}^{(0)}, \,T^{(0)}_b(x)]\\ cov_*[N_{i,b}^{(1)}, \,T^{(1)}_b(x)]
\end{pmatrix} \cdot \begin{pmatrix}
cov_*[N_{i,b}^{(0)}, \,T^{(0)}_b(x)]\\ cov_*[N_{i,b}^{(1)}, \,T^{(1)}_b(x)]
\end{pmatrix}^\top \label{eqn:IJest}
\end{equation}
\end{thm}

\begin{proof}

To complete this proof we rely on subsidiary results proven in sections below. In Lemma \ref{lem:conv} we have shown that
$$
\Sigma_n^{-1/2} \begin{pmatrix} \tilde{F}^{(0)}(x) \\ \tilde{F}^{(1)}(x) \end{pmatrix} \xrightarrow{\DDD} \NNN(\mathbf{0}, \mathbf{I}_2),
$$
where $\Sigma_n$ is the covariance matrix of the Hajek projections of $\tilde{F}^{(0)}(x)$ and $\tilde{F}^{(1)}(x)$. By Lemma \ref{lem:IJest} we can see that each element of $\Sigma_n$ is consistently estimated by the corresponding element of $\hat{\Sigma}_{IJ}$. Then using the Skorohod Representation theorem and the same principles in the proof of Lemma \ref{lem:matrixconsistency} we can say that
$$
\lim_{n \to \infty} \Sigma_n^{-1} \hat{\Sigma}_{IJ} = \lim_{n \to \infty} \hat{\Sigma}_{IJ}^{-1} \Sigma_n = \mathbf{I}_2
$$
We use the condition $\lim\limits_{n \to \infty} \frac{\zeta_{1,k_n}^{(0)}}{\zeta_{1,k_n}^{(1)}} \in (0, \infty)$ for invertibility of $\Sigma_n$. When $(\Sigma_n)_{12} = 0$ we need to have a separate case but the proof for that would follow among the same lines as Lemma \ref{lem:matrixconsistency}. Thus we have shown that $\Sigma_n$ is consistently estimated by $\hat{\Sigma}_{IJ}$.
\end{proof}

A straightforward extension of this result could be for the case of more than one boosting step, provided we define the residuals in a ``noise-free" way similar to $\check{Z}^{(1)}$. So if we define $\check{Z}^{(j)}_i = \left( Y_i - \EE\left[ \sum\limits_{\ell=1}^{j-1} \hat{F}^{(\ell)}(X_i) \right], X_i \right)$, for $j = 1,\dots, m-1$ and impose conditions similar to Condition \ref{cond:regularity} on the forests ($\tilde{F}^{(0)}, \dots, \tilde{F}^{(m)}$) constructed with those datasets then we can get the following result.

\begin{cor}\label{cor:moreboost}
Assume that $\lim\limits_{n \to \infty} \frac{\zeta_{1,k_n}^{(j-1)}}{\zeta_{1,k_n}^{(j)}}$ exists and is finite for all $j = 1,\dots,m$. Then under the same conditions as Theorem \ref{thm:main}, if we construct $m$-step boosted forests with $\tilde{F}^{(0)}, \dots, \tilde{F}^{(m)}$ as the random forests for each stage then
$$
\Sigma_n^{-1/2} \left( \tilde{F}^{(0)}(x), \dots, \tilde{F}^{(m)}(x) \right)^\top \xrightarrow{\DDD} \NNN(\mathbf{0}, \mathbf{I}_{m+1}),
$$
for some some $(m+1) \times (m+1)$ covariance matrix $\Sigma_n$, which can be consistently estimated by the infinitesimal Jackknife covariance matrix estimator given by
$$
\hat{\Sigma}_{IJ} = \sum_{i=1}^n \begin{pmatrix}
cov_*[N_{i,b}^{(0)}, \,T^{(0)}_b(x)]\\ \vdots \\ cov_*[N_{i,b}^{(m)}, \,T^{(m)}_b(x)]
\end{pmatrix} \cdot \begin{pmatrix}
cov_*[N_{i,b}^{(0)}, \,T^{(0)}_b(x)]\\ \vdots \\ cov_*[N_{i,b}^{(m)}, \,T^{(m)}_b(x)]
\end{pmatrix}^\top
$$
\end{cor}

\begin{proof}
That $\Sigma_n^{-1/2} \left( \tilde{F}^{(0)}(x), \dots, \tilde{F}^{(m)}(x) \right)^\top \xrightarrow{\DDD} \NNN(\mathbf{0}, \mathbf{I}_{m+1})$ follows exactly the same arguments as in Lemma \ref{lem:conv}, where $\Sigma_n$ is the covariance matrix of the Hajek projections of $\left( \tilde{F}^{(0)}(x), \dots, \tilde{F}^{(m)}(x) \right)^\top$. Also using very similar arguments in Lemma \ref{lem:IJest} we can show that the elements of $\Sigma_n$ is consistently estimated by the corresponding elements of $\hat{\Sigma}_{IJ}$ and then using the Skorohod Representation theorem and Lemma \ref{lem:matrixconsistency} we conclude that $\hat{\Sigma}_{IJ}$ is a consistent estimator of $\Sigma_n$.
\end{proof}

\begin{lem}\label{lem:conv}
\eqref{eqn:conv} holds true under the conditions of Theorem \ref{thm:main}.
\end{lem}

\begin{proof}
Since we have assumed the Lindeberg Feller type Condition \ref{cond:mainLF} and since $\EE T^4 \leq C < \infty \implies \EE T^2 \leq C_1$ for some $C, C_1$, we can easily see that \eqref{eqn:conv} is simply a bivariate extension of Theorem 1(i) of \cite{mentch2016quantifying}, where the U-statistic $\begin{pmatrix} \tilde{F}^{(0)}(x) \\ \tilde{F}^{(1)}(x) \end{pmatrix}$ is bivariate but have the same kernel $T$ for both its dimensions. The dataset used to construct said U-statistic is $\begin{pmatrix} Z_i^{(0)} \\ Z_i^{(1)} \end{pmatrix}_{i=1}^n$ which are i.i.d since each $Z^{(1)}_i$ is a fixed function of $Z^{(0)}_i$ and $Z_1^{(0)},\dots,Z_n^{(0)} \overset{iid}{\sim} D_{Z^{(0)}}$. Further the number of trees used to build the random forests $\tilde{F}^{(0)}$ and $\tilde{F}^{(1)}$ is $M_n = {n \choose k_n}$ and thus $\frac{n}{M_n} \to 0$. So we can use the same arguments as in Theorem 1(i) of \cite{mentch2016quantifying} and Theorems 11.2 and 12.3 of \cite{van2000asymptotic} to establish \eqref{eqn:conv}.
\end{proof}

\subsection{Consistency of the Infinitesimal Jackknife Covariance Estimator}\label{sec:proofIJ}

\begin{lem}\label{lem:IJest}
Each element of $\Sigma_n$ in \eqref{eqn:conv}, is consistently estimated by the corresponding element of $\hat{\Sigma}_{IJ}$ given by \eqref{eqn:IJest}.
\end{lem}

\begin{proof}
In this proof we shall drop the subscript in $k_n$ for notational simplicity. Also since the test point $x$ is arbitrarily fixed we'll simplify $T(x;Z_1,\dots,Z_k)$ by $T(Z_1,\dots,Z_k)$. We know that
$$
\Sigma_n = Cov \begin{pmatrix} \mathring{\tilde{F}}^{(0)}(x;Z^{(0)}_1,\dots,Z^{(0)}_n) \\ \mathring{\tilde{F}}^{(1)}(x;Z^{(1)}_1,\dots,Z^{(1)}_n) \end{pmatrix},
$$
where $\mathring{(\tikz[baseline=-0.75ex]\draw[fill] (0,0) circle (1.25pt);)}$ denotes the Hajek projection [\cite{hajek1968asymptotic}]. \cite{efron1981jackknife} showed that since the datasets $Z^{(0)}$ and $Z^{(1)}$ are i.i.d. (separately) and the kernel $T$ satisfies $Var[T(Z_1,\dots,Z_k)] = \EE T^2(Z_1,\dots,Z_k) \leq C_1 < \infty$, so there exists functions $T_1,\dots,T_k$ such that
$$
T(Z_1,\dots,Z_k) = \sum_{i=1}^k T_1(Z_i) + \sum_{i<j} T_2(Z_i,Z_j) + \cdots + T_k(Z_1,\dots,Z_k),
$$
where the $2^k-1$ terms on the right hand side are all uncorrelated and have expected value equal to 0 [The first term, i.e., $\EE[T]$ is assumed to be 0]. Then the Hajek projection for $T$ is given by
$$
\mathring{T}(Z_1,\dots,Z_k) = \sum_{i=1}^k T_1(Z_i)
$$
Applying this decomposition to all the individual trees in $\tilde{F}^{(0)}$ and $\tilde{F}^{(1)}$ we can easily deduce that the Hajek projections of the random forests in both stages are
\begin{align*}
\mathring{\tilde{F}}^{(0)}(x;Z^{(0)}_1,\dots,Z^{(0)}_n) &= {n \choose k}^{-1} {n-1 \choose k-1} \sum_{i=1}^n T_1(Z^{(0)}_i) = \frac{k}{n} \sum_{i=1}^n T_1(Z^{(0)}_i) \\
\mathring{\tilde{F}}^{(1)}(x;Z^{(1)}_1,\dots,Z^{(1)}_n) &= {n \choose k}^{-1} {n-1 \choose k-1} \sum_{i=1}^n T_1(Z^{(1)}_i) = \frac{k}{n} \sum_{i=1}^n T_1(Z^{(1)}_i)
\end{align*}
We shall now follow the steps in the proof of Theorem 9 in \cite{wager2017estimation}. Let us define
$$
\Sigma_n = \begin{pmatrix}
\sigma_{00} & \sigma_{01} \\ \sigma_{01} & \sigma_{11}
\end{pmatrix}, \;\text{and}\; \hat{\Sigma}_{IJ} = \begin{pmatrix}
\hat{\sigma}_{00} & \hat{\sigma}_{01} \\ \hat{\sigma}_{01} & \hat{\sigma}_{11}
\end{pmatrix}
$$
Since we assumed $\frac{k}{\sqrt{n}} \to 0$ that implies both $\frac{k(\log n)^d}{n} = \frac{k}{\sqrt{n}} \cdot \frac{(\log n)^d}{\sqrt{n}} \to 0$ and $\frac{n(n-1)}{(n-k)^2} = \frac{1-1/n}{(1-k/n)^2} \to 1$. Also $\EE T^4$ is bounded and hence $\EE T_1^4$ is also bounded. Thus we can apply the WLLN for triangular arrays to $\sigma_{00}^{-1} \frac{k^2}{n^2}\sum_{i=1}^n T_1^2(Z^{(0)}_i)$ and $\sigma_{11}^{-1} \frac{k^2}{n^2}\sum_{i=1}^n T_1^2(Z^{(1)}_i)$.

So all the steps required in the proof of Theorem 9 and Lemma 12 in \cite{wager2017estimation} is satisfied and we can conclude that
$$
\frac{\hat{\sigma}_{00}}{\sigma_{00}} \overset{p}{\to} 1, \;\text{and}\; \frac{\hat{\sigma}_{11}}{\sigma_{11}} \overset{p}{\to} 1
$$
We shall use very similar steps as in the proof of Theorem 9 and Lemma 12 in \cite{wager2017estimation} to show that $\frac{\hat{\sigma}_{01}}{\sigma_{01}} \overset{p}{\to} 1$ or more specifically the equivalent result that
$$
\frac{\hat{\sigma}_{01}}{\sqrt{\sigma_{00}\sigma_{11}}} - \frac{\sigma_{01}}{\sqrt{\sigma_{00}\sigma_{11}}} \overset{p}{\to} 0,
$$
to avoid the special case where $\sigma_{01}$ and hence $\rho_{01} := \frac{\sigma_{01}}{\sqrt{\sigma_{00}\sigma_{11}}}$ is 0. Recall that we assumed $\EE[T] = 0$ and thus we can write $\hat{\sigma}_{01}$ as
\begin{align}
\hat{\sigma}_{01} &= \frac{k^2}{n^2} \sum_{i=1}^n \EE_{U^{(0)} \subset \hat{D}^{(0)}} \left[ T \,\middle\vert\, U^{(0)}_1 = Z^{(0)}_i \right] \cdot \EE_{U^{(1)} \subset \hat{D}^{(1)}} \left[ T \,\middle\vert\, U^{(1)}_1 = Z^{(1)}_i \right] \nonumber\\
&= \frac{k^2}{n^2} \sum_{i=1}^n (A_i+R_i)(B_i+S_i), \label{eqn:covexpansion}
\end{align}
Here $\hat{D}^{(0)}$ is the empirical distribution over $(Z^{(0)}_1,\dots,Z^{(0)}_n)$ and $\hat{D}^{(1)}$ is the same over $(Z^{(1)}_1,\dots,Z^{(1)}_n)$. Also $U^{(j)}$ is a sample of size $k$ from $\hat{D}^{(j)}$ without replacement, $j = 0,1$ and
\begin{align*}
A_i = \EE_{U^{(0)} \subset \hat{D}^{(0)}} \left[ \mathring{T} \,\middle\vert\, U^{(0)}_1 = Z^{(0)}_i \right], \;\; & R_i = \EE_{U^{(0)} \subset \hat{D}^{(0)}} \left[ T-\mathring{T} \,\middle\vert\, U^{(0)}_1 = Z^{(0)}_i \right]\\
B_i = \EE_{U^{(1)} \subset \hat{D}^{(1)}} \left[ \mathring{T} \,\middle\vert\, U^{(1)}_1 = Z^{(1)}_i \right], \;\; & S_i = \EE_{U^{(1)} \subset \hat{D}^{(1)}} \left[ T-\mathring{T} \,\middle\vert\, U^{(1)}_1 = Z^{(1)}_i \right]
\end{align*}
Now Lemma 13 of \cite{wager2017estimation} shows that
\begin{equation}
\frac1{\sigma_{00}} \frac{k^2}{n^2} \sum_{i=1}^n R_i^2 \overset{p}{\to} 0 \qquad\text{and}\qquad \frac1{\sigma_{11}} \frac{k^2}{n^2} \sum_{i=1}^n S_i^2 \overset{p}{\to} 0 \label{eqn:crossterms}
\end{equation}
So if we apply \eqref{eqn:crossterms} to the four term expansion of $\frac{\hat{\sigma}_{01}}{\sqrt{\sigma_{00}\sigma_{11}}}$ using \eqref{eqn:covexpansion} then the last three terms vanish by Cauchy-Schwartz. For the first term we need further calculations. We can write
\begin{align*}
A_i &= \EE_{U^{(0)} \subset \hat{D}^{(0)}} \left[ \mathring{T} \,\middle\vert\, U^{(0)}_1 = Z^{(0)}_i \right] \\
&= \left( 1- \frac{k}{n} \right) T_1(Z^{(0)}_i) + \left( \frac{k-1}{n-1}-\frac{k}{n} \right) \sum_{j \neq i} T_1(Z^{(0)}_j) \\
&= \frac{n-k}{n} \left[ T_1(Z^{(0)}_i) -\frac1{n-1}\sum_{j \neq i} T_1(Z^{(0)}_j) \right] \\
\text{Similarly}\; B_i &= \EE_{U^{(1)} \subset \hat{D}^{(1)}} \left[ \mathring{T} \,\middle\vert\, U^{(1)}_1 = Z^{(1)}_i \right] \\
&= \frac{n-k}{n} \left[ T_1(Z^{(1)}_i) -\frac1{n-1}\sum_{j \neq i} T_1(Z^{(1)}_j) \right]
\end{align*}
Define $C_{01} = Cov\left( T_1(Z^{(0)}_i),T_1(Z^{(1)}_i) \right)$ (same for all $i$), and note that
$$
\EE(A_iB_i) = \left(\frac{n-k}{n}\right)^2 \left[ C_{01} + \frac{n-1}{(n-1)^2}\; C_{01} \right]= \left(\frac{n-k}{n}\right)^2 \frac{n}{n-1}\; C_{01}
$$
The cross terms vanish since $Z^{(1)}_i$ is a fixed function of $Z^{(0)}_i$ and hence uncorrelated with $Z^{(0)}_j$ if $j \neq i$. Then it is seen that
\begin{align*}
\EE\left[ \frac{n-1}{n} \left(\frac{n}{n-k}\right)^2 \frac{k^2}{n^2} \sum_{i=1}^n A_iB_i \right] &= \frac{k^2}{n} C_{01} \\
&= \frac{k}{n} Cov \left( \mathring{T}(Z^{(0)}_1,\dots,Z^{(0)}_k),\mathring{T}(Z^{(1)}_1,\dots,Z^{(1)}_k) \right) \\
&= Cov \left( \mathring{\tilde{F}}^{(0)}(x;Z^{(0)}_1,\dots,Z^{(0)}_n),
\mathring{\tilde{F}}^{(1)}(x;Z^{(1)}_1,\dots,Z^{(1)}_n) \right) = \sigma_{01}
\end{align*}
Again since $\frac{n(n-1)}{(n-k)^2} = \frac{1-1/n}{(1-k/n)^2} \to 1$ and $\EE T^4$ is bounded we can apply WLLN for triangular arrays to conclude that
$$
\left(\frac1{\sqrt{\sigma_{00}\sigma_{11}}} \frac{k^2}{n^2} \sum_{i=1}^n A_iB_i\right) - \rho_{01} \overset{p}{\to} 0 \implies \frac{\hat{\sigma}_{01}}{\sqrt{\sigma_{00}\sigma_{11}}} - \rho_{01} \overset{p}{\to} 0
$$

\end{proof}

\begin{lem}\label{lem:matrixconsistency}
If $A^{(n)}$ and $B^{(n)}$ are two sequences of square matrices of fixed size $m \times m$, each of which are invertible, such that
$$
\lim_{n \to \infty} \frac{A^{(n)}_{ij}}{B^{(n)}_{ij}} \text{ exists and equals }\alpha \neq 0, \,\forall\, 1\leq i,j \leq m,
$$
then the matrices also satisfy
$$
\lim_{n \to \infty} \left(A^{(n)}\right)^{-1} B^{(n)} = \alpha \mathbf{I}_m, \;\text{and}\; \lim\limits_{n \to \infty} \left(B^{(n)}\right)^{-1} A^{(n)} = (1/\alpha) \mathbf{I}_m
$$
\end{lem}

\begin{proof}
Let $\mathbf{A}^{(n)}_{ij}$ be the cofactors of $A^{(n)}$. Then we know that $\left( A^{(n)} \right)^{-1} = \frac1{\lvert A^{(n)} \rvert} \left(\left( \mathbf{A}^{(n)}_{ij} \right)\right)^\top$. So if we define $C^{(n)} := \left(A^{(n)}\right)^{-1} B^{(n)}$ then it satisfies
\begin{align*}
\lim_{n \to \infty} C^{(n)}_{ij} &= \lim_{n \to \infty} \frac1{\lvert A^{(n)} \rvert} \sum_{k=1}^m \mathbf{A}^{(n)}_{ki} B^{(n)}_{kj} \\
&= \lim_{n \to \infty} \frac1{\lvert A^{(n)} \rvert} \sum_{k=1}^m \mathbf{A}^{(n)}_{ki} A^{(n)}_{kj} \cdot \frac{B^{(n)}_{kj}}{A^{(n)}_{kj}} \\
&= \alpha \ind_{\{i=j\}} \\
\implies \lim_{n \to \infty} C^{(n)} &= \alpha \mathbf{I}_m
\end{align*}
Similarly it can be easily shown that $\lim\limits_{n \to \infty} \left(B^{(n)}\right)^{-1} A^{(n)} = (1/\alpha) \mathbf{I}_m$
\end{proof}

\subsection{Origins of the Infinitesimal Jackknife Covariance Estimator}

We called the consistent estimator given by \eqref{eqn:IJest} to be an infinitesimal Jackknife estimator. It includes estimates for variance and covariance terms. In fact the infinitesimal Jackknife estimator for the variance term has been defined in Theorem 1 of \cite{efron2014estimation}, which in turn borrows heavily from an older definition in Chapter 6 of \cite{efron1982jackknife}. In both those references the result was established with the assumption that we were using bootstrapped data. For bootstrap data we take a sample of the \textit{same} size as the original dataset \textit{with} replacement, as opposed to a sample of size \textit{smaller} than the original dataset \textit{without} replacement which is used in our work and in \cite{wager2017estimation}. That the variance estimates originally defined for the former case applies consistently in the latter case as well was shown by \cite{wager2017estimation}. Our work aims to show the same for the covariance estimate as well. For that the last result left to establish is that the covariance estimate we defined as the off-diagonal element in \eqref{eqn:IJest} is indeed a two-sample analogue to the definition of the variance estimate in Theorem 1 of \cite{efron2014estimation}.

\begin{lem}\label{lem:covestisIJ}
Let $Y^{(0)} = (Y^{(0)}_1,\dots,Y^{(0)}_n)$ and $Y^{(1)} = (Y^{(1)}_1,\dots,Y^{(1)}_n)$ be two i.i.d. datasets from separate distributions and let $T_0$ and $T_1$ be two estimators. Then we define the ideal smooth bootstrap statistics corresponding to $T_0$ and $T_1$ by
\begin{equation}
S_j(Y^{(j)}) = \frac1B \sum_{b=1}^B T_j(Y^{(j)}_b), \;\text{for } j = 0,1, \label{eqn:bootestdef}
\end{equation}
where the average is done over all possible $B = n^n$ bootstrap samples $Y^{(j)}_b = (Y^{(j)}_{b1},\dots,Y^{(j)}_{bn})$, $j = 0,1$. Then the infinitesimal Jackknife estimate of covariance between $S_0$ and $S_1$ is given by
\begin{equation}
\widetilde{Cov}(S_0, S_1) = \sum_{i=1}^n \prod_{j=0}^1 cov_*\left[N^{(j)}\{i,b\}, \,T_j(Y^{(j)}_b)\right], \;\; b = 1,\dots,B \label{eqn:covIJdef}
\end{equation}
where $N^{(j)}\{i,b\} = \#\{Y^{(j)}_i \in Y^{(j)}_b\}$ for $i \in [n], b \in [B], j = 0,1$.
\end{lem}

\textit{Note:} The infinitesimal Jackknife estimate is referred to as the non-parametric delta-method estimate in \cite{efron2014estimation}

\begin{proof}
Let $P^0 = \left( \frac1n, \cdots, \frac1n \right)^\top \in \RR^n$ and let $P^{(1)}$ and $P^{(2)}$ be independent copies of $P^* \sim \frac{\text{Multinomial}_n(n,P^0)}{n}$. We can immediately establish that
$$
\EE(P^*) = P^0, \;\Sigma(P^*) = \frac{I_n}{n^2} - \frac{P^0(P^0)^\top}{n},
$$
where $\Sigma$ denotes the covariance matrix. Define $\theta_j(P^{(j)}) = T_j(F_j(Y^{(j)}))$, for $j = 0,1$, where $F_j$ is a distribution on $Y^{(j)}$ which has mass $P^{(j)}_i$ at $Y^{(j)}_i$. Then we can easily see that $S_j(Y^{(j)}) = \EE(\theta(P^{(j)}))$ for $j = 0,1$. Now analogous to (6.15) and (6.16) in \cite{efron1982jackknife} we can define the infinitesimal Jackknife estimate of the covariance between $S_0$ and $S_1$ by
\begin{align}
\widetilde{Cov}(S_0,S_1) &= Cov(\hat{\theta}_0(P^{(0)}),\hat{\theta}_1(P^{(1)})), && \label{eqn:covdef}\\
\text{where } \hat{\theta}_j(P^{(j)}) &= \theta_j(P^0) + (P^{(j)} - P_0)^\top U^{(j)},\;j = 0,1 \label{eqn:estdef}\\
U^{(j)}_i &= \lim_{\epsilon \to 0} \frac{\theta_j(P_0 + \epsilon(e_i - P_0)) - \theta_j(P_0)}{\epsilon},\;i = 1,\dots,n \nonumber
\end{align}
Here $e_i = (0,\dots,0,1,0,\dots,0)$ is the $i$th coordinate vector in $\RR^n$. Now $\theta_0$ and $\theta_1$ are only defined in the simplex $\PP_n = \{P \in \RR^n: P_i \geq 0, \sum_{i=1}^n P_i = 1\}$. We can do a homogeneous extension of the definition to an open set $\QQ_n \supset \PP_n$ by $\theta_j(Q) := \theta_j(Q / \sum_{i=1}^n Q_i)$ for $Q \in \QQ_n$. Then we can define gradients $D^{(j)}$ at $P_0$ by
$$
D^{(j)}_i = \left.\frac{\partial}{\partial P_i}\theta_j(P) \right|_{P = P^0}, \;\text{for}\; i = 1,\dots,n, \;j = 0,1
$$
It is easily seen that the directional derivatives $U^{(j)}_i = (e_i-P^0)^\top D^{(j)}$, and hence $(P^0)^\top U^{(j)} = \frac1n\sum_{i=1}^n U^{(j)}_i = 0$ for $j = 0,1$. So using \eqref{eqn:estdef} we have
\begin{equation}
Cov(\hat{\theta}_0(P^{(0)}),\hat{\theta}_1(P^{(1)})) = (U^{(0)})^\top \Sigma(P^*) U^{(1)} = \frac1{n^2} \sum_{i=1}^n U^{(0)}_i U^{(1)}_i \label{eqn:covbypartial}
\end{equation}
Now (3.21) in the proof of Theorem 1 of \cite{efron2014estimation} shows that
\begin{equation}
U^{(j)}_i = n \cdot cov_*\left[N^{(j)}\{i,b\}, \,T_j(Y^{(j)}_b)\right], \;\; b = 1,\dots,B, i = 1,\dots,n, j = 0,1 \label{eqn:borrowedEfron}
\end{equation}
Substituting \eqref{eqn:borrowedEfron} in \eqref{eqn:covbypartial} and that result in \eqref{eqn:covdef} we get \eqref{eqn:covIJdef}.
\end{proof}

To apply Lemma \ref{lem:covestisIJ} in practice, we need to estimate $S_0$ and $S_1$ by using the same formula as \eqref{eqn:bootestdef} and then estimate of the covariance between them by \eqref{eqn:covIJdef}, the only difference being that $B$ will be the number of bootstrap samples we use instead of all $n^n$ possible bootstrap samples.

\section{Further Empirical Studies} \label{sec:moreresults}

We first address the discrepancy between our theoretical notation and the practical coding usage as mentioned briefly at the start of \cref{sec:results}. In \eqref{eqn:firststage} and \eqref{eqn:secondstage} we wrote forests as $\frac1B \sum_{(I)} w_IT_I$ in which the $w_I$ were independent Bernoulli random variables and $\sum_{(I)} \EE[w_I] = B$. If instead we select exactly $B$ trees at random we write the resulting forest as $\frac1B \sum_{(I)} v_IT_I$ where $v_I$ are random but with the constraint that $\sum_{(I)} v_I = B$. Let $M = {n \choose k}$. Then
\begin{align*}
\text{if}\; I \neq J, \;\text{then}\;& \EE(v_Iv_J) = \PP(v_I = v_J = 1) = {M-2 \choose B-2}/{M \choose B} = \frac{B(B-1)}{M(M-1)},\\
& \EE(w_Iv_J) = \PP(w_I = v_J = 1) = \frac{B^2}{M^2} = \EE(w_Iw_J)\\
\text{if}\; I = J, \;\text{then}\;& \EE(w_Iw_J) = \PP(w_I = 1) = \frac{B}{M} = \EE(v_Iv_J),\\
& \EE(w_Iv_J) = \PP(w_I = v_J = 1) = \frac{B^2}{M^2}
\end{align*}
Now note that $w_I$ and $T_I$ are independent. If $K = \EE[T_I^2]$ we can show that
\begin{align*}
&\EE\left[ \frac1B \sum_{(I)} w_IT_I - \frac1B \sum_{(I)} v_IT_I \right]^2 \\
=& \frac1{B^2} \EE\left[ \frac1B \sum_{(I)} (w_I-v_I)T_I \right]^2 \\
=& \frac1{B^2} \cdot \EE[T_I^2] \cdot \EE \left[ \sum_{(I)}\sum_{(J)} (w_I-v_I)(w_J-v_J) \right] \\
=& \frac{K}{B^2} \left[ M\left( \frac{2B}{M} - \frac{2B^2}{M^2} \right) + 2{M \choose 2}\left( \frac{B(B-1)}{M(M-1)} + \frac{B^2}{M^2} - \frac{2B^2}{M^2} \right) \right] \\
=& \frac{K}{B^2} \left[ 2B - \frac{2B^2}{M} + B(B-1) - \frac{B^2(M-1)}{M} \right] = K \left[\frac1B - \frac1M \right]
\end{align*}
Since $K$ is assumed to be bounded in the assumptions of Theorem \ref{thm:normal} we see that under large values of $B$ and $M$, i.e., asymptotically the two selection schemes are equivalent.

\subsection{Exploring coverage of the BoostedForest algorithm over a distribution of covariates}

Here we briefly look into how the coverage intervals (the value of the statistic C.C. in Table \ref{table:simresults}) for random forests and the two boosted forest variants behave when considered over a distribution of the covariate space. We are considering the same model as in \cref{sec:simresults}, i.e., $Y = \sum_{i=1}^5 X_i + \epsilon$, where $X \sim U([-1,1]^{15})$ and $\epsilon \sim N(0,1)$. We randomly sampled 1000 points from $U([-1,1]^{15})$ and checked the percentage of times the signal is included in the coverage interval over 1000 iterations of the algorithms. In Figure \ref{fig:covhist} below we take a look ot the histogram of this percentage for different variants and different number of trees (the vertical line marks 95\%).

\begin{figure}[ht]
\centering
\includegraphics[width=.6\textheight]{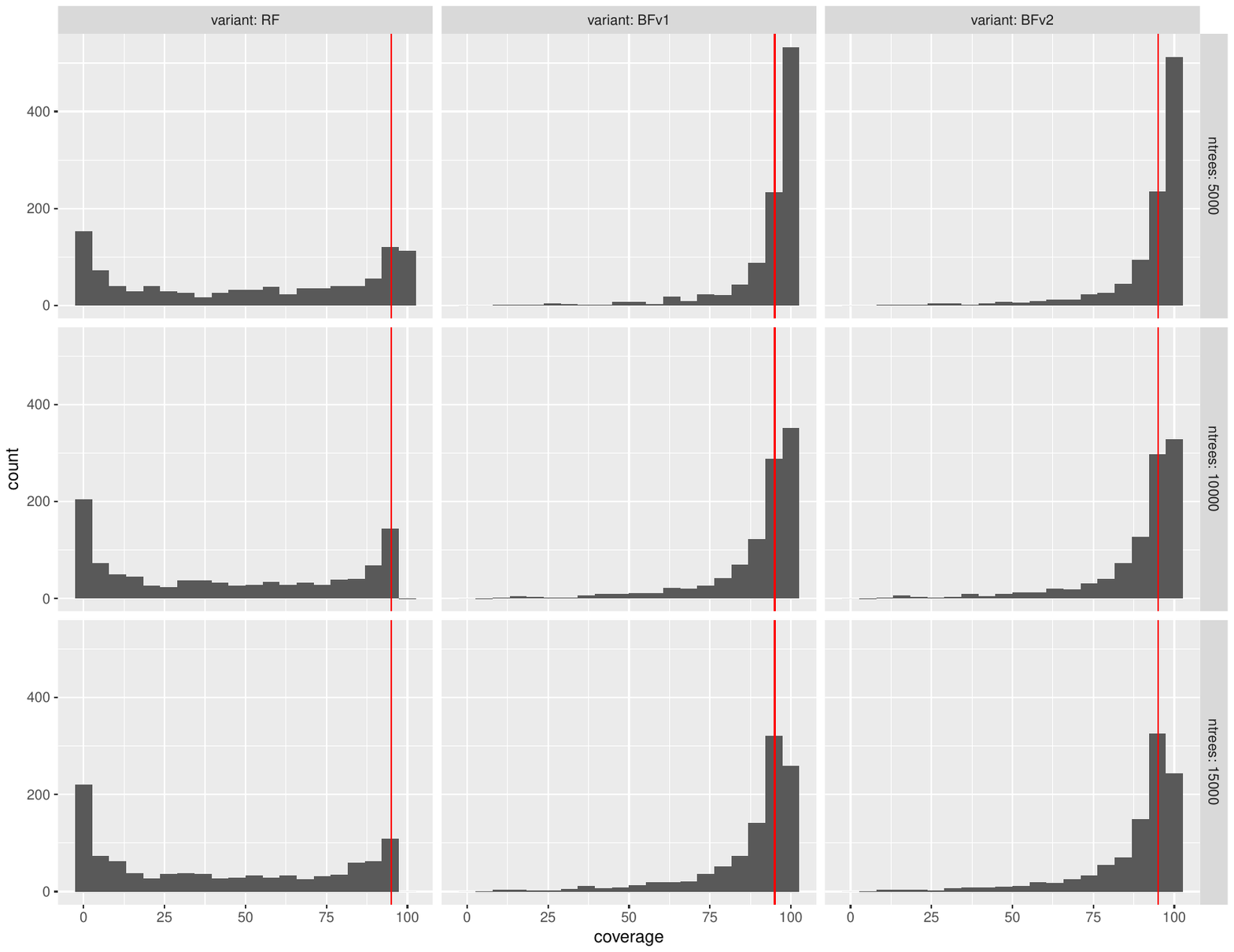}
\caption{\small Histogram of empirical coverage probabilities from 1000 randomly sampled covariate values. Tiled along columns by the randomforest (RF) and both variants of the boostedforest algorithms (BFv1 and BFv2) and along rows by differing number of trees.}
\label{fig:covhist}
\end{figure}

We see that while the base random forest algorithm is not that great at coverage, both the boosted forest algorithms are really good at it. This is to be expected as mentioned before in the paper, since boosting leads to better ``centering'' of the coverage interval thereby increasing the chances of the signal lying inside the interval. We also note that with lower number of trees the coverage intervals are slightly more conservative leading to higher that 95\% coverage but with increasing numbers of trees this problem is mitigated.

We can further explore this with Figure \ref{fig:covsignal} where we plot the coverage probabilities vs the value of the signal. 

\begin{figure}[ht]
\centering
\includegraphics[width=.6\textheight]{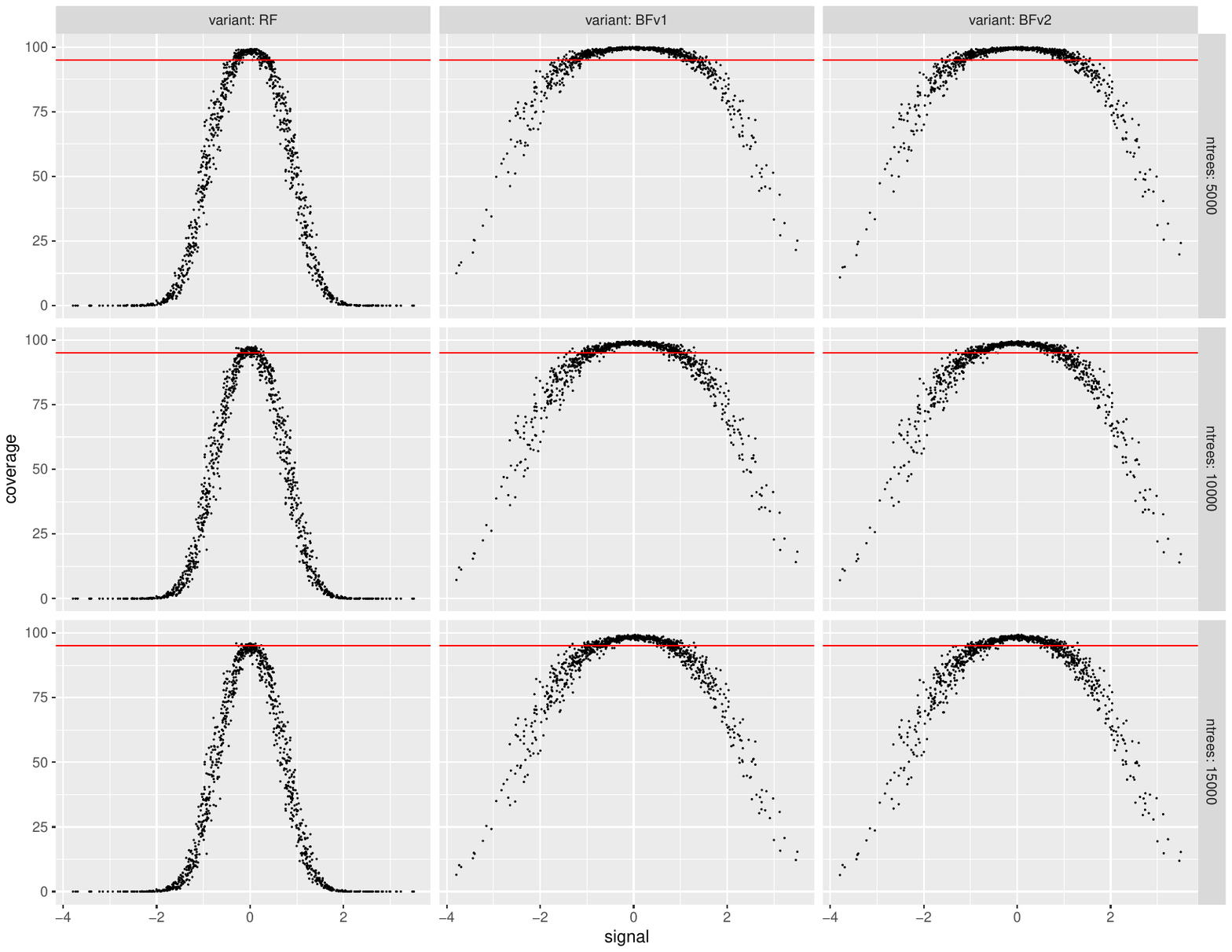}
\caption{\small Plot of empirical coverage probabilities vs the value of the signal for 1000 randomly sampled covariate values. The figure is tiled columnwise by the randomforest (RF) and boostedforest algorithms (BFv1 and BFv2) and along rows by differing number of trees.}
\label{fig:covsignal}
\end{figure}

Notice that for the base random forest the coverage is pretty good for a narrow band of signal values near 0 and dropping pretty quickly on either end of that band. In contrast both variants of boosted forest provide good coverage for a quite bigger band of signal values; this band gets a bit narrower for increasing number of trees since the coverage intervals gets more precise due to lower variance estimates.

\subsection{Performance on Simulated Datasets with High Noise} \label{sec:simresultsnoisy}

In this section we present simulation studies on the performance of the One-Step Boosted Forest algorithm as noise variance changes as noted at the end of \cref{sec:realresults}. The setup and metrics evaluated are almost exactly the same as in \cref{sec:realresults}. Our model now is
$$
Y = \sum_{i=1}^5 x_i + \epsilon, \;\epsilon \sim N(0,\sigma^2)
$$
and the values of $\sigma^2$ can be $\frac12$, $\frac13$ or $\frac14$. We have also added in results from the case $\sigma^2 = 1$ as in \cref{sec:realresults} and Table \ref{table:simresults}. The combined figures are in Table \ref{table:simresultsnoisy}. We only consider Variant II of the boosted forest since the two variants performed similarly and this had lower bias and lower variance. 

Table \ref{table:simresultsnoisy} shows behaviour similar to Table \ref{table:simresults}. We see that performance gets better as the test point moves further away from the origin, the ``center'' of the dataset. The increment of the number of trees in the forest helps the performance as does the decrease in noise - both are as expected. Similar changes can be noticed in the figures for percentage coverage of confidence intervals, i.e., they become worse as the test point moves further away from the origin, the number of trees decreases and the noise increases although in all cases they remain conservative; while the opposite is observed for the bias and the Kolmogorv-Smirnov statistics. Variance estimates are fairly constant for all test points but they of course decrease with both the number of trees and as the noise decreases.

\afterpage{\clearpage
\begin{sidewaystable}[ht]
\setlength{\tabcolsep}{4.5pt}
\centering

\begin{tabular}{|r|r||rrrrr||rrrrr||rrrrr|}
\hline
& & \multicolumn{5}{c||}{$B=5000$} & \multicolumn{5}{c||}{$B=10000$} & \multicolumn{5}{c|}{$B=15000$} \\
\cline{3-7} \cline{8-12} \cline{13-17}
  &   & $p_1$ & $p_2$ & $p_3$ & $p_4$ & $p_5$ & $p_1$ & $p_2$ & $p_3$ & $p_4$ & $p_5$ & $p_1$ & $p_2$ & $p_3$ & $p_4$ & $p_5$\\
\hhline{*{17}{=}} \rule{0pt}{3ex}
 & $\overline{\text{Bias}}$ & -0.0047 & -0.0172 & -0.0175 & -0.0421 & -0.0900 & -0.0057 & -0.0180 & -0.0180 & -0.0433 & -0.0906 & -0.0056 & -0.0178 & -0.0178 & -0.0425 & -0.0896\\ \rule{0pt}{3ex}
 & $\overline{\widehat{V}_{IJ}}$ & 0.0853 & 0.0856 & 0.0846 & 0.0834 & 0.0822 & 0.0692 & 0.0695 & 0.0686 & 0.0671 & 0.0660 & 0.0637 & 0.0641 & 0.0632 & 0.0618 & 0.0606\\ \rule{0pt}{3ex}
 & $\frac{\overline{\widehat{V}_{IJ}}}{V(\widehat{F})}$ & 2.1328 & 2.1757 & 2.1121 & 1.9443 & 1.8344 & 1.7398 & 1.7839 & 1.7061 & 1.5827 & 1.4913 & 1.5968 & 1.6482 & 1.5695 & 1.4611 & 1.3772\\ \rule{0pt}{3ex}
 & K.S. & 0.0943 & 0.1197 & 0.1029 & 0.1415 & 0.1905 & 0.0734 & 0.0993 & 0.0825 & 0.1158 & 0.1829 & 0.0614 & 0.0937 & 0.0730 & 0.1089 & 0.1830\\ \rule{0pt}{3ex}
 & C.C. & 99.5 & 99.1 & 99.9 & 99.0 & 98.1 & 98.7 & 98.4 & 98.9 & 98.1 & 97.0 & 98.1 & 98.2 & 98.6 & 97.4 & 96.2\\ \rule{0pt}{3ex}
\multirow{-6}{*}{\raggedleft\arraybackslash \rotatebox{90}{$\sigma^2=1$}} & P.I. & -70.74 & -13.71 & -0.05 & 52.94 & 74.10 & -69.44 & -12.57 & -0.07 & 53.33 & 74.29 & -70.12 & -12.96 & -0.33 & 53.48 & 74.48\\
\cline{1-17} \rule{0pt}{3ex}

 & $\overline{\text{Bias}}$ & -0.0042 & -0.0118 & -0.0080 & -0.0245 & -0.0661 & -0.0049 & -0.0122 & -0.0082 & -0.0247 & -0.0657 & -0.0045 & -0.0125 & -0.0082 & -0.0250 & -0.0659\\ \rule{0pt}{3ex}
 & $\overline{\widehat{V}_{IJ}}$ & 0.0608 & 0.0607 & 0.0600 & 0.0582 & 0.0568 & 0.0492 & 0.0492 & 0.0484 & 0.0469 & 0.0456 & 0.0454 & 0.0452 & 0.0446 & 0.0432 & 0.0417\\ \rule{0pt}{3ex}
 & $\frac{\overline{\widehat{V}_{IJ}}}{V(\widehat{F})}$ & 2.2480 & 2.3103 & 2.1626 & 2.0332 & 1.9153 & 1.8080 & 1.8855 & 1.7530 & 1.6466 & 1.5476 & 1.6729 & 1.7385 & 1.6213 & 1.5190 & 1.4247\\ \rule{0pt}{3ex}
 & K.S. & 0.0961 & 0.1252 & 0.1056 & 0.1311 & 0.1882 & 0.0726 & 0.1020 & 0.0784 & 0.1104 & 0.1792 & 0.0635 & 0.0966 & 0.0685 & 0.1060 & 0.1743\\ \rule{0pt}{3ex}
 & C.C. & 99.7 & 99.5 & 99.8 & 99.1 & 98.1 & 98.7 & 99.1 & 99.5 & 98.3 & 97.2 & 98.3 & 98.8 & 98.9 & 97.7 & 96.8\\ \rule{0pt}{3ex}
\multirow{-6}{*}{\raggedleft\arraybackslash \rotatebox{90}{$\sigma^2=1/2$}} & P.I. & -44.17 & 9.44 & 18.23 & 64.53 & 81.53 & -44.50 & 9.84 & 18.46 & 64.74 & 81.66 & -44.26 & 9.98 & 18.61 & 64.85 & 81.73\\
\cline{1-17} \rule{0pt}{3ex}

 & $\overline{\text{Bias}}$ & -0.0035 & -0.0102 & -0.0042 & -0.0173 & -0.0567 & -0.0037 & -0.0103 & -0.0038 & -0.0174 & -0.0565 & -0.0035 & -0.0101 & -0.0041 & -0.0178 & -0.0569\\ \rule{0pt}{3ex}
 & $\overline{\widehat{V}_{IJ}}$ & 0.0523 & 0.0522 & 0.0515 & 0.0500 & 0.0484 & 0.0423 & 0.0422 & 0.0416 & 0.0402 & 0.0386 & 0.0390 & 0.0388 & 0.0383 & 0.0369 & 0.0354\\ \rule{0pt}{3ex}
 & $\frac{\overline{\widehat{V}_{IJ}}}{V(\widehat{F})}$ & 2.3098 & 2.4134 & 2.2220 & 2.1071 & 1.9923 & 1.8483 & 1.9459 & 1.7795 & 1.6874 & 1.5853 & 1.7262 & 1.8128 & 1.6606 & 1.5606 & 1.4576\\ \rule{0pt}{3ex}
 & K.S. & 0.1057 & 0.1260 & 0.1032 & 0.1258 & 0.1838 & 0.0764 & 0.1027 & 0.0751 & 0.1068 & 0.1705 & 0.0699 & 0.0953 & 0.0665 & 0.0988 & 0.1692\\ \rule{0pt}{3ex}
 & C.C. & 99.6 & 99.6 & 99.9 & 99.3 & 98.9 & 98.8 & 99.3 & 99.4 & 98.6 & 97.6 & 98.6 & 99.2 & 99.0 & 98.0 & 96.9\\ \rule{0pt}{3ex}
\multirow{-6}{*}{\raggedleft\arraybackslash \rotatebox{90}{$\sigma^2=1/3$}} & P.I. & -29.94 & 20.63 & 27.01 & 69.26 & 84.48 & -31.19 & 20.51 & 26.32 & 69.14 & 84.47 & -30.53 & 21.01 & 26.98 & 69.32 & 84.49\\
\cline{1-17} \rule{0pt}{3ex}

 & $\overline{\text{Bias}}$ & -0.0034 & -0.0092 & -0.0025 & -0.0140 & -0.0518 & -0.0031 & -0.0091 & -0.0021 & -0.0146 & -0.0526 & -0.0035 & -0.0094 & -0.0022 & -0.0147 & -0.0524\\ \rule{0pt}{3ex}
 & $\overline{\widehat{V}_{IJ}}$ & 0.0480 & 0.0478 & 0.0472 & 0.0456 & 0.0438 & 0.0388 & 0.0385 & 0.0381 & 0.0366 & 0.0350 & 0.0358 & 0.0355 & 0.0351 & 0.0336 & 0.0321\\ \rule{0pt}{3ex}
 & $\frac{\overline{\widehat{V}_{IJ}}}{V(\widehat{F})}$ & 2.3211 & 2.4562 & 2.2530 & 2.1364 & 2.0187 & 1.8974 & 2.0031 & 1.8336 & 1.7267 & 1.6147 & 1.7421 & 1.8462 & 1.6850 & 1.5939 & 1.4848\\ \rule{0pt}{3ex}
 & K.S. & 0.1010 & 0.1261 & 0.1019 & 0.1213 & 0.1765 & 0.0808 & 0.1061 & 0.0797 & 0.1017 & 0.1736 & 0.0727 & 0.0933 & 0.0670 & 0.0970 & 0.1655\\ \rule{0pt}{3ex}
 & C.C. & 99.6 & 99.7 & 99.9 & 99.2 & 98.9 & 98.9 & 99.5 & 99.4 & 98.5 & 97.5 & 98.3 & 99.2 & 98.8 & 98.3 & 97.0\\ \rule{0pt}{3ex}
\multirow{-6}{*}{\raggedleft\arraybackslash \rotatebox{90}{$\sigma^2=1/4$}} & P.I. & -24.45 & 26.07 & 31.16 & 71.65 & 85.93 & -22.87 & 26.62 & 31.53 & 71.75 & 85.91 & -23.83 & 26.57 & 31.22 & 71.88 & 85.97\\
\hline
\end{tabular}
\caption{\it Comparison of Variant II of the One-Step Boosted Forest with respect to Random Forests in case of noisy data and a linear signal. The metrics evaluated are the same as in Table \ref{table:simresults} in \cref{sec:simresults}.}
\label{table:simresultsnoisy}
\end{sidewaystable}}

\subsection{Performance on Simulated Datasets with Nonlinear response} \label{sec:simresultsnorm}

We conducted further simulations with a nonlinear response function, namely, the norm. Our training features still come uniformly from $[-1,1]^{15}$ but now our model is
$$
Y = \|x\|_2 + \epsilon, \;\epsilon \sim N(0,\sigma^2)
$$
where we consider different values of $\sigma^2$ in $\{1,5,10,15\}$. The maximum standard deviation value of $\sqrt{15}$ was so chosen since the norm can have maximum value $\sqrt{15}$ in our domain.

We present below two tables with our results. Table \ref{table:simresultsnorm} has the figures for $\sigma^2 = 1$. Most of the patterns from the numbers in Table \ref{table:simresults} can also be observed here. But in this case we should note that the signal has a derivative discontinuity at 0 which is likely to create a region of high bias at 0; moving further from this point should reduce that but we also expect bias at the edges of the covariate range. From the table we see that for all test points $p_i$ the bias decreases significantly for both variants of the One-Step Boosted Forest algorithm when compared to (unboosted) random forests although there is no change in bias as the number of trees in the forest grows. Variance estimates are also seen to be increasing with boosting and decreasing with number of trees; although it doesn't change with distance of the test points from the origin. The Kolmogorov-Smirnov statistics are significantly higher than those in Table \ref{table:simresults}. Also the prediction improvements (i.e., decrease in MSE) are markedly better in this case since bias can decreases significantly by boosting (although it still remains quite high). Variant II improvement is also slightly better than Variant I in most cases. Note that in all cases even though bias is reduced substantially by boosting it is still much higher than the estimated variances. So the 95\% confidence intervals will not be accurate here and were hence not reported in the table.

\afterpage{\clearpage
\begin{sidewaystable}[ht]
\centering

\begin{tabular}{|c|r||ccccc||ccccc||ccccc|}
 \hline\rule{0pt}{3ex}
 & & \multicolumn{5}{c||}{$B = 5000$} & \multicolumn{5}{c||}{$B = 10000$} & \multicolumn{5}{c|}{$B = 15000$}\\ \hhline{*{2}{|~}|*{15}{|-}|}\rule{0pt}{4ex}
 & & $\overline{\text{Bias}}$ & $\overline{\widehat{V}_{IJ}}$ & $\frac{\overline{\widehat{V}_{IJ}}}{V(\widehat{F})}$ & K.S. & P.I.(\%) & $\overline{\text{Bias}}$ & $\overline{\widehat{V}_{IJ}}$ & $\frac{\overline{\widehat{V}_{IJ}}}{V(\widehat{F})}$ & K.S. & P.I.(\%) & $\overline{\text{Bias}}$ & $\overline{\widehat{V}_{IJ}}$ & $\frac{\overline{\widehat{V}_{IJ}}}{V(\widehat{F})}$ & K.S. & P.I.(\%) \\
 \hhline{*{17}{=}}\rule{0pt}{3ex}
 & RF & 1.8074 & 0.0165 & 1.660 & 1.000 & 0.00 & 1.8073 & 0.0119 & 1.196 & 1.000 & 0.00 & 1.8076 & 0.0103 & 1.039 & 1.000 & 0.00\\ \cline{2-2}\rule{0pt}{3ex}
 $p_1$ & BFv1 & 1.5217 & 0.0497 & 1.845 & 1.000 & 28.51 & 1.5208 & 0.0393 & 1.455 & 1.000 & 28.58 & 1.5213 & 0.0357 & 1.329 & 1.000 & 28.56\\ \cline{2-2}\rule{0pt}{3ex}
 & BFv2 & 1.5209 & 0.0472 & 1.755 & 1.000 & 28.58 & 1.5209 & 0.0381 & 1.414 & 1.000 & 28.57 & 1.5214 & 0.0349 & 1.292 & 1.000 & 28.55\\ \hline\rule{0pt}{3ex}

 & RF & 1.4837 & 0.0167 & 1.610 & 1.000 & 0.00 & 1.4836 & 0.0120 & 1.159 & 1.000 & 0.00 & 1.4839 & 0.0105 & 1.008 & 1.000 & 0.00\\ \cline{2-2}\rule{0pt}{3ex}
 $p_2$ & BFv1 & 1.2051 & 0.0503 & 1.782 & 0.995 & 33.06 & 1.2042 & 0.0397 & 1.411 & 0.997 & 33.15 & 1.2047 & 0.0362 & 1.287 & 0.997 & 33.13\\ \cline{2-2}\rule{0pt}{3ex}
 & BFv2 & 1.2044 & 0.0478 & 1.699 & 0.996 & 33.15 & 1.2044 & 0.0385 & 1.368 & 0.997 & 33.13 & 1.2046 & 0.0353 & 1.253 & 0.997 & 33.13\\ \hline\rule{0pt}{3ex}

 & RF & 1.4852 & 0.0166 & 1.625 & 1.000 & 0.00 & 1.4850 & 0.0120 & 1.181 & 1.000 & 0.00 & 1.4852 & 0.0104 & 1.029 & 1.000 & 0.00\\ \cline{2-2}\rule{0pt}{3ex}
 $p_3$ & BFv1 & 1.2068 & 0.0502 & 1.806 & 0.996 & 33.03 & 1.2065 & 0.0395 & 1.434 & 0.997 & 33.06 & 1.2071 & 0.0360 & 1.307 & 0.998 & 33.00\\ \cline{2-2}\rule{0pt}{3ex}
 & BFv2 & 1.2064 & 0.0476 & 1.710 & 0.996 & 33.07 & 1.2067 & 0.0383 & 1.385 & 0.998 & 33.03 & 1.2069 & 0.0352 & 1.276 & 0.998 & 33.02\\ \hline\rule{0pt}{3ex}

 & RF & 1.1779 & 0.0171 & 1.549 & 1.000 & 0.00 & 1.1778 & 0.0123 & 1.117 & 1.000 & 0.00 & 1.1779 & 0.0107 & 0.973 & 1.000 & 0.00\\ \cline{2-2}\rule{0pt}{3ex}
 $p_4$ & BFv1 & 0.9193 & 0.0515 & 1.735 & 0.979 & 37.44 & 0.9189 & 0.0406 & 1.372 & 0.985 & 37.49 & 0.9193 & 0.0369 & 1.252 & 0.988 & 37.46\\ \cline{2-2}\rule{0pt}{3ex}
 & BFv2 & 0.9186 & 0.0488 & 1.640 & 0.980 & 37.53 & 0.9188 & 0.0392 & 1.326 & 0.986 & 37.51 & 0.9188 & 0.0360 & 1.218 & 0.989 & 37.52\\ \hline\rule{0pt}{3ex}

 & RF & 0.8872 & 0.0179 & 1.453 & 1.000 & 0.00 & 0.8872 & 0.0128 & 1.049 & 1.000 & 0.00 & 0.8873 & 0.0112 & 0.913 & 1.000 & 0.00\\ \cline{2-2}\rule{0pt}{3ex}
 $p_5$ & BFv1 & 0.6594 & 0.0536 & 1.630 & 0.881 & 41.51 & 0.6596 & 0.0423 & 1.295 & 0.904 & 41.49 & 0.6597 & 0.0385 & 1.179 & 0.910 & 41.48\\ \cline{2-2}\rule{0pt}{3ex}
 & BFv2 & 0.6593 & 0.0510 & 1.552 & 0.887 & 41.51 & 0.6594 & 0.0410 & 1.250 & 0.906 & 41.50 & 0.6595 & 0.0376 & 1.148 & 0.912 & 41.50\\
 \hline
\end{tabular}
\caption{\it Comparison of the two variants of the One-Step Boosted Forest (shorthands {\normalfont BFv1} and {\normalfont BFv2} for the two variants respectively) with respect to Random Forests (shorthands {\normalfont RF}) in case of a nonlinear signal. The metrics evaluated are the same as in Table \ref{table:simresults} in \cref{sec:simresults} with the exception of percentage coverage of confidence intervals.}
\label{table:simresultsnorm}
\end{sidewaystable}}

The combined figures for all the noise levels $\sigma^2 = 1, 5, 10, 15$ are in Table \ref{table:simresultsnoisynorm} although they are only for Variant II of boosted forests. The performance metrics show similar behaviour here as compared to Table \ref{table:simresultsnoisy}, i.e., it becomes worse as the noise increases but also better as the number of trees increases. So in conclusion the boosted forest algorithm performs best when training dataset is not too noisy, the number of trees in the forests for each stage are high and the test point is far from the ``center'' so it has a moderate amount of bias if predicted by the base random forest. On the other hand, the boosted forest performs worst in a setting where we're using a sparse forest and a very noisy training dataset to predict at a test point which has a very high bias to begin with (boosting might decrease the bias slightly but the increase in variance estimate offsets it).

\afterpage{\clearpage
\begin{sidewaystable}[ht]
\centering
\begin{tabular}{|c|c||ccccc||ccccc||ccccc|}
 \hline\rule{0pt}{3ex}
 $Y|X = x \sim$ & & \multicolumn{5}{c||}{$B = 5000$} & \multicolumn{5}{c||}{$B = 10000$} & \multicolumn{5}{c|}{$B = 15000$}\\ \hhline{*{2}{|~}|*{15}{|-}|}\rule{0pt}{4ex}
 $\NNN(\|x\|_2, \sigma^2)$ & & $\overline{\text{Bias}}$ & $\overline{\widehat{V}_{IJ}}$ & $\frac{\overline{\widehat{V}_{IJ}}}{V(\widehat{F})}$ & K.S. & P.I.(\%) & $\overline{\text{Bias}}$ & $\overline{\widehat{V}_{IJ}}$ & $\frac{\overline{\widehat{V}_{IJ}}}{V(\widehat{F})}$ & K.S. & P.I.(\%) & $\overline{\text{Bias}}$ & $\overline{\widehat{V}_{IJ}}$ & $\frac{\overline{\widehat{V}_{IJ}}}{V(\widehat{F})}$ & K.S. & P.I.(\%) \\
 \hhline{*{17}{=}}\rule{0pt}{3ex}
  & $p_1$ & 1.5209 & 0.0472 & 1.755 & 1.000 & 28.58 & 1.5209 & 0.0381 & 1.414 & 1.000 & 28.57 & 1.5214 & 0.0349 & 1.292 & 1.000 & 28.55\\\rule{0pt}{3ex}
  & $p_2$ & 1.2044 & 0.0478 & 1.699 & 0.996 & 33.15 & 1.2044 & 0.0385 & 1.368 & 0.997 & 33.13 & 1.2046 & 0.0353 & 1.253 & 0.997 & 33.13\\ \rule{0pt}{3ex}
  $\sigma^2 = 1$ & $p_3$ & 1.2064 & 0.0476 & 1.710 & 0.996 & 33.07 & 1.2067 & 0.0383 & 1.385 & 0.998 & 33.03 & 1.2069 & 0.0352 & 1.276 & 0.998 & 33.02\\ \rule{0pt}{3ex}
  & $p_4$ & 0.9186 & 0.0488 & 1.640 & 0.980 & 37.53 & 0.9188 & 0.0392 & 1.326 & 0.986 & 37.51 & 0.9188 & 0.0360 & 1.218 & 0.989 & 37.52\\ \rule{0pt}{3ex}
  & $p_5$ & 0.6593 & 0.0510 & 1.552 & 0.887 & 41.51 & 0.6594 & 0.0410 & 1.250 & 0.906 & 41.50 & 0.6595 & 0.0376 & 1.148 & 0.912 & 41.50\\ \hline\rule{0pt}{3ex}

  & $p_1$ & 1.5088 & 0.2328 & 1.830 & 0.926 & 26.66 & 1.5093 & 0.1865 & 1.452 & 0.942 & 26.55 & 1.5102 & 0.1711 & 1.338 & 0.948 & 26.48\\ \rule{0pt}{3ex}
  & $p_2$ & 1.1946 & 0.2353 & 1.770 & 0.840 & 29.88 & 1.1945 & 0.1883 & 1.409 & 0.859 & 29.82 & 1.1953 & 0.1730 & 1.297 & 0.870 & 29.73\\ \rule{0pt}{3ex}
  $\sigma^2 = 5$ & $p_3$ & 1.1975 & 0.2345 & 1.779 & 0.841 & 29.75 & 1.1976 & 0.1875 & 1.425 & 0.864 & 29.72 & 1.1983 & 0.1721 & 1.316 & 0.874 & 29.69\\ \rule{0pt}{3ex}
  & $p_4$ & 0.9114 & 0.2393 & 1.705 & 0.712 & 31.72 & 0.9114 & 0.1920 & 1.371 & 0.740 & 31.77 & 0.9107 & 0.1762 & 1.268 & 0.749 & 31.88\\ \rule{0pt}{3ex}
  & $p_5$ & 0.6489 & 0.2484 & 1.605 & 0.543 & 30.69 & 0.6481 & 0.1996 & 1.303 & 0.567 & 30.92 & 0.6488 & 0.1832 & 1.196 & 0.581 & 30.81\\ \hline\rule{0pt}{3ex}

  & $p_1$ & 1.5001 & 0.4647 & 1.819 & 0.792 & 24.16 & 1.5017 & 0.3724 & 1.459 & 0.821 & 24.08 & 1.5030 & 0.3415 & 1.345 & 0.831 & 23.95\\ \rule{0pt}{3ex}
  & $p_2$ & 1.1868 & 0.4692 & 1.764 & 0.675 & 25.93 & 1.1886 & 0.3759 & 1.413 & 0.702 & 25.76 & 1.1901 & 0.3449 & 1.302 & 0.717 & 25.61\\ \rule{0pt}{3ex}
  $\sigma^2 = 10$ & $p_3$ & 1.1919 & 0.4681 & 1.767 & 0.687 & 25.72 & 1.1937 & 0.3742 & 1.433 & 0.708 & 25.65 & 1.1950 & 0.3436 & 1.319 & 0.722 & 25.51\\ \rule{0pt}{3ex}
  & $p_4$ & 0.9065 & 0.4772 & 1.723 & 0.551 & 25.21 & 0.9053 & 0.3823 & 1.377 & 0.572 & 25.18 & 0.9079 & 0.3512 & 1.273 & 0.580 & 24.96\\ \rule{0pt}{3ex}
  & $p_5$ & 0.6434 & 0.4966 & 1.621 & 0.415 & 18.24 & 0.6412 & 0.3985 & 1.314 & 0.435 & 18.69 & 0.6434 & 0.3659 & 1.201 & 0.442 & 18.33\\ \hline\rule{0pt}{3ex}

  & $p_1$ & 1.5005 & 0.6956 & 1.809 & 0.702 & 21.09 & 1.4988 & 0.5585 & 1.468 & 0.732 & 21.43 & 1.5003 & 0.5111 & 1.350 & 0.744 & 21.34\\ \rule{0pt}{3ex}
  & $p_2$ & 1.1874 & 0.7020 & 1.750 & 0.593 & 21.30 & 1.1863 & 0.5643 & 1.420 & 0.614 & 21.65 & 1.1870 & 0.5166 & 1.306 & 0.628 & 21.60\\ \rule{0pt}{3ex}
  $\sigma^2 = 15$ & $p_3$ & 1.1927 & 0.6988 & 1.778 & 0.591 & 21.39 & 1.1935 & 0.5613 & 1.440 & 0.615 & 21.58 & 1.1938 & 0.5152 & 1.327 & 0.626 & 21.61\\ \rule{0pt}{3ex}
  & $p_4$ & 0.9081 & 0.7163 & 1.728 & 0.468 & 18.13 & 0.9056 & 0.5744 & 1.392 & 0.481 & 18.55 & 0.9074 & 0.5270 & 1.279 & 0.491 & 18.40\\ \rule{0pt}{3ex}
  & $p_5$ & 0.6414 & 0.7448 & 1.630 & 0.352 & 6.86 & 0.6403 & 0.5995 & 1.322 & 0.370 & 7.29 & 0.6408 & 0.5483 & 1.211 & 0.371 & 7.35\\ 
 \hline
\end{tabular}
\caption{\it Comparison of Variant II of the One-Step Boosted Forest with respect to Random Forests in case of a dataset with nonlinear signal and varying levels of noise. The metrics evaluated are the same as in Table \ref{table:simresults} in \cref{sec:simresults}.}
\label{table:simresultsnoisynorm}
\end{sidewaystable}}

\subsection{Comparing Out-of-bag vs Inbag Residuals vs Bootstrapped Forests} \label{sec:oobvsother}

In light of the results in Theorem 1 of \cite{efron2014estimation} and Lemma \ref{lem:covestisIJ} above we can see that a different construction of the One-Step Boosted Forest would result in a variance estimate very similar to the ones in \cref{sec:varxest}. In the construction of the One-Step Boosted Forest (\cref{sec:basicalgo}) we used subsampling without replacement to construct the individual trees in the forest. Instead of that we could also use full bootstrapped resampling to do the same. We shall focus only on the case where we take different bootstrap resamples in the two stages, i.e., the case analogous to Variant II of the One-Step Boosted Forest. In that case the same calculations as in \eqref{eqn:bfindvar} and \eqref{eqn:bfindvarest} will hold. The construction of the One-Step Boosted Forests by bootstrap resampling can be detailed as follows.

Assume that we have $B$ trees in the forests in each stage. Now the total number of possible bootstrap resamples is $n^n$. So we can define the first and second stage forests by
\begin{align*}
\hat{F}^{(j)}(x) &= \frac1B \sum_{b=1}^B T(x; Z_{I_b^{(j)}}^{(j)}) = \frac1{n^n} \sum_{I \in \BBB_n} w_I^{(j)} T(x; Z_{I}^{(j)}),\; j = 0,1,
\end{align*}
where the values of $j$ correspond to each stage, each $w_I^{(j)}$ is a binary random variable taking the value $n^n/B$ with probability $B/n^n$ and the value 0 with probability $1 - B/n^n$ and $\BBB_n$ is the set of all indices under full bootstrap resampling. Then $c = Var(w_I^{(j)}) = n^n/B$.

So if we follow exactly the same arguments as in \eqref{eqn:bfindvar} and \eqref{eqn:bfindvarest} by replacing ${n \choose k}$ with $n^n$, and also applying Theorem 1 of \cite{efron2014estimation} and Lemma \ref{lem:covestisIJ}  we get that the variance estimator for Variant II the One-Step Boosted Forest constructed with bootstrap resamples is exactly the same formula as \eqref{eqn:bfindvarest}. The only difference being that $N_{i,b}^{(j)}$ are no longer binary variables but can take any non-negative integer value corresponding to the number of times the $i$\textsuperscript{th} datapoint is included in the $b$\textsuperscript{th} bootstrap resample during the construction of the $j$\textsuperscript{th} stage.

Also note the two following ways in which we can construct the response for fitting the boosted stage of the forest.
\begin{enumerate}[(a)]
\item taking out-of-bag residuals as response for the second stage (this is the default way we've used is all previous tables in this paper), or
\item taking inbag residuals as response for the second stage
\end{enumerate}
We had discussed this briefly at the beginning of \cref{sec:results}. We combine these two ways of extracting residuals with two ways of constructing the trees
\begin{enumerate}[(a)]
\item using independent subsamples in both stages of the forests (Variant II of the boosted forest algorithm)
\item using independent full bootstrapped resamples in both stages of the forests as we discussed above.
\end{enumerate}

In Table \ref{table:simresultsboot} we present a comparison using simulated datasets of these four new methods against a single random forest. It is also believed that a single forest with full bootstrap resamples for each tree performs better than one with subsamples and so we have included some data regarding that comparison as well marked with the column ``RFboot". The setup is exactly the same as the one that leads to the numbers in Table \ref{table:simresultsnorm} but we consider only the case where we have $B = 15000$ trees in each forest. 

\begin{table}[ht]
\centering
\begin{tabular}{|c|c||cccccc|}
 \hline\rule{0pt}{4ex}
 & & RF & RFboot & BFv2  & BFv2    & BFboot & BFboot  \\
 & &    &        & (oob) & (inbag) & (oob)  & (inbag) \\
 \hhline{*{8}{=}}\rule{0pt}{3ex}

 \multirow{5}{*}{$p_1$} & $\overline{\text{Bias}}$ & 1.8076 & 1.6575 & 1.5214 & 1.5739 & 1.3224 & 1.5247\\ \rule{0pt}{3ex}
 & $\overline{\widehat{V}_{IJ}}$ & 00.0103 & 0.0353 & 0.0349 & 0.0288 & 0.0948 & 0.0508\\ \rule{0pt}{3ex}
 & $\frac{\overline{\widehat{V}_{IJ}}}{V(\widehat{F})}$ & 1.0389 & 1.7607 & 1.2923 & 1.2444 & 1.7152 & 1.6116\\ \rule{0pt}{3ex}
 & K.S. & 1.0000 & 1.0000 & 1.0000 & 1.0000 & 0.9800 & 0.9978\\ \rule{0pt}{3ex}
 & P.I. & 0 & 15.57 & 28.55 & 23.71 & 44.96 & 28.11\\ \hline\rule{0pt}{3ex}

 \multirow{5}{*}{$p_2$} & $\overline{\text{Bias}}$ & 1.4839 & 1.3373 & 1.2046 & 1.2559 & 1.0106 & 1.2077\\ \rule{0pt}{3ex}
 & $\overline{\widehat{V}_{IJ}}$ & 0.0105 & 0.0359 & 0.0353 & 0.0292 & 0.0959 & 0.0515\\ \rule{0pt}{3ex}
 & $\frac{\overline{\widehat{V}_{IJ}}}{V(\widehat{F})}$ & 1.0084 & 1.7132 & 1.2526 & 1.2052 & 1.6575 & 1.5631\\ \rule{0pt}{3ex}
 & K.S. & 1.0000 & 0.9980 & 0.9971 & 0.9992 & 0.9254 & 0.9932\\ \rule{0pt}{3ex}
 & P.I. & 0 & 18.22 & 33.13 & 27.61 & 51.22 & 32.58\\ \hline\rule{0pt}{3ex}

 \multirow{5}{*}{$p_3$} & $\overline{\text{Bias}}$ & 1.4852 & 1.3402 & 1.2069 & 1.2580 & 1.0160 & 1.2118\\ \rule{0pt}{3ex}
 & $\overline{\widehat{V}_{IJ}}$ & 0.0104 & 0.0357 & 0.0352 & 0.0290 & 0.0956 & 0.0512\\ \rule{0pt}{3ex}
 & $\frac{\overline{\widehat{V}_{IJ}}}{V(\widehat{F})}$ & 1.0293 & 1.7504 & 1.2762 & 1.2327 & 1.6833 & 1.5920\\ \rule{0pt}{3ex}
 & K.S. & 1.0000 & 0.9990 & 0.9977 & 0.9995 & 0.9263 & 0.9934\\ \rule{0pt}{3ex}
 & P.I. & 0 & 18.03 & 33.02 & 27.52 & 50.86 & 32.28\\ \hline\rule{0pt}{3ex}

 \multirow{5}{*}{$p_4$} & $\overline{\text{Bias}}$ & 1.1779 & 1.0448 & 0.9188 & 0.9663 & 0.7449 & 0.9258\\ \rule{0pt}{3ex}
 & $\overline{\widehat{V}_{IJ}}$ & 0.0107 & 0.0369 & 0.0360 & 0.0297 & 0.0982 & 0.0528\\ \rule{0pt}{3ex}
 & $\frac{\overline{\widehat{V}_{IJ}}}{V(\widehat{F})}$ & 0.9729 & 1.7008 & 1.2180 & 1.1736 & 1.6552 & 1.5569\\ \rule{0pt}{3ex}
 & K.S. & 1.0000 & 0.9970 & 0.9891 & 0.9957 & 0.8130 & 0.9718\\ \rule{0pt}{3ex}
 & P.I. & 0 & 20.39 & 37.52 & 31.43 & 56.08 & 36.29\\ \hline\rule{0pt}{3ex}

 \multirow{5}{*}{$p_5$} & $\overline{\text{Bias}}$ & 0.8873 & 0.7692 & 0.6595 & 0.7008 & 0.5052 & 0.6641\\ \rule{0pt}{3ex}
 & $\overline{\widehat{V}_{IJ}}$ & 0.0112 & 0.0389 & 0.0376 & 0.0311 & 0.1028 & 0.0556\\ \rule{0pt}{3ex}
 & $\frac{\overline{\widehat{V}_{IJ}}}{V(\widehat{F})}$ & 0.9133 & 1.5314 & 1.1483 & 1.1055 & 1.4944 & 1.4049\\ \rule{0pt}{3ex}
 & K.S. & 1.0000 & 0.9627 & 0.9120 & 0.9515 & 0.6247 & 0.8642\\ \rule{0pt}{3ex}
 & P.I. & 0 & 22.82 & 41.50 & 35.06 & 59.48 & 39.89\\ \hline
\end{tabular}
\caption{\it Comparison of the different ways of constructing Variant II of the One-Step Boosted Forest with respect to Random Forests. The metrics evaluated are the same as in Table \ref{table:simresults} in \cref{sec:simresults}.}
\label{table:simresultsboot}
\end{table}

We see that all the cases gives us significant improvement over the base algorithm of random forest. Looking at the performance improvement metric (decrease in MSE) the single forest with full bootstrap resamples do not perform as well as any of the boosted forest variants. Out of the other variants the boosted forest constructed with out-of-bag residuals and full bootstrapped resamples perform the best followed by the one with out-of-bag residuals but subsamples. The former also reduces the bias the most but its variance is very high thereby reducing its overall performance. It is also computationally more expensive than using subsamples. The boosted forest with inbag residuals has the lowest variance as hypothesised at the beginning of \cref{sec:results} but it fails to reduce the bias as much as the other methods and hence it performs the worst among the 4 variants.

We have not been able to extend the logic in \cref{sec:asympnormal} and \cref{sec:proofnormal} to show that the estimate constructed with bootstrap resamples follow normality. Actually the unusually high values for the Kolmogorov-Smirnov statistics (K.S.) in Table \ref{table:simresultsboot} shows that asymptotic normality likely doesn't hold true.

In Table \ref{table:realresultsboot} we can see the results of this comparison using the same 11 real datasets from the UCI repository and the same setup and parameters as used in \cref{sec:realresults}. For this table we have dropped the ``RFboot" and "BFboot (inbag)" comparisons.

\begin{table}[ht]
\centering
\begin{tabular}{|l||ccc||cccc|}
  \hline\rule{0pt}{3ex}
  \multirow{3}{*}{Dataset} & BFoob & BFinbag & BFboot & RF & BFoob & BFinbag & BFboot \\ \cline{2-8}\rule{0pt}{3ex}
  & \multicolumn{3}{c||}{\multirow{2}{*}{Improvement}} & \multicolumn{4}{c|}{PI length} \\ \rule{0pt}{3ex}
  & \multicolumn{3}{c||}{} & \multicolumn{4}{c|}{PI Coverage} \\
  \hline\hline\rule{0pt}{3ex}
  \multirow{2}{*}{Yacht-hydrodynamics} & \multirow{2}{*}{82.04} & \multirow{2}{*}{76.71} & \multirow{2}{*}{93.60} & 14.25 & 15.22 & 15.10 & 18.53 \\ \rule{0pt}{3ex}
  & & & & 91.33 & 99.00 & 99.33 & 100.00 \\
  \hline\rule{0pt}{3ex}
  \multirow{2}{*}{BikeSharing-hour} & \multirow{2}{*}{73.99} & \multirow{2}{*}{70.94} & \multirow{2}{*}{89.10} & 1.48 & 1.50 & 1.50 & 2.99 \\ \rule{0pt}{3ex}
  & & & & 100.00 & 100.00 & 100.00 & 100.00 \\
  \hline\rule{0pt}{3ex}
  \multirow{2}{*}{Concrete} & \multirow{2}{*}{52.20} & \multirow{2}{*}{47.29} & \multirow{2}{*}{69.54} & 25.29 & 27.33 & 26.80 & 39.37 \\ \rule{0pt}{3ex}
  & & & & 96.02 & 98.93 & 98.64 & 99.81 \\
  \hline\rule{0pt}{3ex}
  \multirow{2}{*}{Airfoil} & \multirow{2}{*}{43.65} & \multirow{2}{*}{41.40} & \multirow{2}{*}{60.96} & 14.21 & 15.25 & 15.11 & 22.03 \\ \rule{0pt}{3ex}
  & & & & 94.33 & 99.07 & 99.20 & 99.93 \\
  \hline\rule{0pt}{3ex}
  \multirow{2}{*}{Boston-housing} & \multirow{2}{*}{26.22} & \multirow{2}{*}{21.86} & \multirow{2}{*}{31.74} & 0.61 & 0.64 & 0.63 & 0.82 \\ \rule{0pt}{3ex}
  & & & & 95.80 & 97.00 & 96.20 & 99.80 \\
  \hline\rule{0pt}{3ex}
  \multirow{2}{*}{Auto-mpg} & \multirow{2}{*}{20.79} & \multirow{2}{*}{19.85} & \multirow{2}{*}{30.28} & 11.40 & 11.86 & 11.78 & 14.56 \\ \rule{0pt}{3ex}
  & & & & 93.85 & 95.90 & 95.90 & 98.72 \\
  \hline\rule{0pt}{3ex}
  \multirow{2}{*}{Wine-quality-white} & \multirow{2}{*}{11.42} & \multirow{2}{*}{9.95} & \multirow{2}{*}{22.91} & 3.34 & 3.96 & 3.77 & 7.31 \\ \rule{0pt}{3ex}
  & & & & 98.45 & 99.45 & 99.28 & 100.00 \\
  \hline\rule{0pt}{3ex}
  \multirow{2}{*}{Parkinsons} & \multirow{2}{*}{8.09} & \multirow{2}{*}{7.14} & \multirow{2}{*}{13.76} & 35.63 & 42.35 & 40.45 & 87.76 \\ \rule{0pt}{3ex}
  & & & & 99.71 & 99.95 & 99.90 & 100.00 \\
  \hline\rule{0pt}{3ex}
  \multirow{2}{*}{Wine-quality-red} & \multirow{2}{*}{7.45} & \multirow{2}{*}{6.70} & \multirow{2}{*}{15.98} & 2.59 & 2.85 & 2.78 & 4.46 \\ \rule{0pt}{3ex}
  & & & & 95.91 & 97.48 & 97.23 & 99.94 \\
  \hline\rule{0pt}{3ex}
  \multirow{2}{*}{SkillCraft} & \multirow{2}{*}{4.30} & \multirow{2}{*}{4.01} & \multirow{2}{*}{2.30} & 4.37 & 5.00 & 4.81 & 9.53 \\ \rule{0pt}{3ex}
  & & & & 98.74 & 99.52 & 99.40 & 99.97 \\
  \hline\rule{0pt}{3ex}
  \multirow{2}{*}{Communities} & \multirow{2}{*}{3.05} & \multirow{2}{*}{2.95} & \multirow{2}{*}{2.12} & 0.62 & 0.68 & 0.66 & 1.12 \\ \rule{0pt}{3ex}
  & & & & 96.88 & 98.09 & 97.99 & 99.55 \\
  \hline
\end{tabular}
\caption{\it Comparison of the three ways of constructing Variant II of the One-Step Boosted Forest. We use shorthands for PI = Prediction Interval \eqref{eqn:predint}, RF = Random Forests, BFoob, BFinbag and BFboot for BFv2 (oob), BFv2 (inbag) and BFboot (oob) respectively from Table \ref{table:simresultsboot}. The Improvement and PI Coverage figures are in percentages.}
\label{table:realresultsboot}
\end{table}

We see that the results are similar to that from the simulation dataset. The boosted forest constructed with bootstrapped resamples perform the best in all the cases. But it has higher variance estimate which contributes to the higher prediction interval length and coverage. Thus the prediction interval given by that method might be too conservative. This method is also computationally extremely expensive. The boosted forest with inbag residuals has a low variance as hypothesised near the beginning of \cref{sec:results} but it also does not reduce bias more than the other 2 methods leading to its bad performance.

In our opinion the boosted forest with out-of-bag residuals strikes a good balance in terms of decent performance. It also has a variance estimate that isn't as low as the boosted forest with inbag residuals but also isn't as high as the boosted forest constructed with bootstrapped resamples, leading to a wide but not too conservative prediction interval. Also it is computationally competitive with the boosted forest with inbag residuals but much cheaper compared to the boosted forest constructed with bootstrapped resamples.


\end{document}